%% file: main.tex
\documentclass{article}
\usepackage{enumitem}
\usepackage[utf8]{inputenc} %
\usepackage[T1]{fontenc}    %
\usepackage{hyperref}       %
\usepackage{url}            %
\usepackage{booktabs}       %
\usepackage{amsfonts}       %
\usepackage{nicefrac}       %
\usepackage{microtype}      %
\usepackage{graphicx}       %
\usepackage[numbers,sort]{natbib}
\usepackage{multirow}
\usepackage{bbding}
\usepackage{makecell}
\graphicspath{{media/}}     %
\usepackage{amsmath}
\usepackage{xcolor}
\usepackage{subfigure}

\usepackage{threeparttable}
\usepackage{amssymb}%
\usepackage{pifont}%
\definecolor{bgcolor}{rgb}{0.93,0.99,1}

\PassOptionsToPackage{numbers, compress, sort}{natbib}
\usepackage[final]{neurips_2023}

\DeclareMathOperator{\Tr}{Tr}
\newcommand{\ee}{\mathbb{E}}
\newcommand{\mb}{\boldsymbol}
\newcommand{\eg}{\emph{e.g.}}
\newcommand{\tabincell}[2]{
\begin{tabular}{@{}#1@{}}#2\end{tabular}
}

\newlength{\myfootnotespace}
\newcommand\freefootnote[1]{%
  \let\thefootnote\relax%
  \footnotetext{\hspace{\myfootnotespace}#1}%
  \let\thefootnote\svthefootnote%
}

\input{math_commands}

\author{%
  Jing Xu$^*$   \\
  IIIS, Tsinghua University \\
  \texttt{xujing21@mails.tsinghua.edu.cn} \\
  \And
  Jiaye Teng$^*$  \\
  IIIS, Tsinghua University \\
  \texttt{tjy20@mails.tsinghua.edu.cn} 
  \And
  Yang Yuan \\
  IIIS, Tsinghua University \\
  Shanghai Artificial Intelligence Laboratory \\
  Shanghai Qi Zhi Institute \\
  \texttt{yuanyang@tsinghua.edu.cn}
  \And
  Andrew Chi-Chih Yao \\
  IIIS, Tsinghua University \\
  Shanghai Artificial Intelligence Laboratory \\
  Shanghai Qi Zhi Institute \\
  \texttt{andrewcyao@tsinghua.edu.cn}\\ 
}

\begin{document}

\title{Towards Data-Algorithm Dependent Generalization:\\ a Case Study on Overparameterized Linear Regression}

\date{}

\maketitle

\input{text/abstract}
\input{text/intro}
\input{text/relatedWork}

\input{text/preliminary}

\input{text/mainResults}

\input{text/experiments}
\input{text/conclusion}
\input{text/ack}

\bibliographystyle{plainnat}
\bibliography{sample}

\newpage
\appendix

\input{text/appendix_proof}

\input{text/appendix_example}

\input{text/appendix_experiment}

\end{document}

%% file: math_commands.tex
\usepackage{hyperref}       %
\usepackage{amsfonts}
\usepackage{amsmath}
\usepackage{amsthm}
\usepackage{bm}
\usepackage{float}
\usepackage{amssymb}
\usepackage{graphicx} 
\usepackage{nccmath}
\usepackage{thm-restate}
\usepackage{color}
\usepackage{cleveref}

\newcommand{\cA}{\mathcal{A}}

\newcommand{\cD}{\mathcal{D}}

\newcommand{\cH}{\mathcal{H}}

\newcommand{\cR}{\mathcal{R}}

\newcommand{\bE}{\mathbb{E}}

\newcommand{\bR}{\mathbb{R}}

\usepackage{mathtools}
\usepackage{amssymb}
\usepackage{latexsym}
\usepackage{dsfont}
\usepackage{mathrsfs}
\usepackage{amssymb}
\usepackage{amsfonts}
\usepackage{bm}
\usepackage{xspace}

        {\medskip}

        {\hspace*{\fill}$\Box$\par\vspace{4mm}}
        {\hspace*{\fill}$\Box$\par}

\newcommand*{\argmin}{\mathop{\mathrm{argmin}}}

\newcommand{\Nbb}{\mathbb{N}}

\ifx\BlackBox\undefined
\newcommand{\BlackBox}{\rule{1.5ex}{1.5ex}}  %
\fi

\ifx\QED\undefined
\def\QED{~\rule[-1pt]{5pt}{5pt}\par\medskip}
\fi

\ifx\theorem\undefined
\newtheorem{theorem}{Theorem}[section]
\fi

\ifx\example\undefined
\newtheorem{example}{Example}[section]
\fi

\ifx\property\undefined
\newtheorem{property}{Property}
\fi

\ifx\lemma\undefined
\newtheorem{lemma}{Lemma}[section]
\fi

\ifx\proposition\undefined
\newtheorem{proposition}{Proposition}[section]
\fi

\ifx\remark\undefined
\newtheorem{remark}{Remark}
\fi

\ifx\corollary\undefined
\newtheorem{corollary}{Corollary}[section]
\fi

\ifx\definition\undefined
\newtheorem{definition}{Definition}[section]
\fi

\ifx\conjecture\undefined
\newtheorem{conjecture}{Conjecture}[section]
\fi

\ifx\axiom\undefined
\newtheorem{axiom}[theorem]{Axiom}
\fi

\ifx\claim\undefined
\newtheorem{claim}[theorem]{Claim}
\fi

\ifx\assumption\undefined
\newtheorem{assumption}{Assumption}
\fi

\ifx\problem\undefined
\newtheorem{problem}{Problem}[section]
\fi

\ifx\fact\undefined
\newtheorem{fact}{Fact}
\fi

%% file: text/abstract.tex
\begin{abstract}

One of the major open problems in machine learning is to characterize generalization in the overparameterized regime, where most traditional generalization bounds become inconsistent even for overparameterized linear regression~\citep{DBLP:conf/nips/NagarajanK19}. In many scenarios, this failure can be attributed to obscuring the crucial interplay between the training algorithm and the underlying data distribution.
This paper demonstrate that the generalization behavior of overparameterized model should be analyzed in a both data-relevant and algorithm-relevant manner. 
To make a formal characterization,
We introduce a notion called data-algorithm compatibility, which considers the generalization behavior of the entire data-dependent training trajectory, instead of traditional last-iterate analysis. 
We validate our claim by studying the setting of solving overparameterized linear regression with gradient descent.
Specifically, we perform a data-dependent trajectory analysis and derive a sufficient condition for compatibility in such a setting. Our theoretical results demonstrate that if we take early stopping iterates into consideration, generalization can hold with significantly weaker restrictions on the problem instance than the previous last-iterate analysis.\freefootnote{$^*$Equal Contribution}
\end{abstract}

%% file: text/intro.tex
\section{Introduction}
Although deep neural networks achieve great success in practice \citep{DBLP:journals/nature/SilverSSAHGHBLB17,DBLP:conf/naacl/DevlinCLT19,DBLP:conf/nips/BrownMRSKDNSSAA20}, their remarkable generalization ability is still among the essential mysteries in the deep learning community. 
One of the most intriguing features of deep neural networks is overparameterization, which confers a level of tractability to the training problem, but leaves traditional generalization theories failing to work.
In generalization analysis, both the training algorithm and the data distribution play essential roles~\citep{goldt2019dynamics, DBLP:conf/iclr/JiangNMKB20}.
For instance, a line of work~\citep{DBLP:journals/cacm/ZhangBHRV21,DBLP:conf/nips/NagarajanK19} highlights the role of the algorithm by showing that the algorithm-irrelevant uniform convergence bounds can become inconsistent
in deep learning regimes. Another line of work~\citep{bartlett2020benign,tsigler2020benign} on benign overfitting emphasizes the role of data distribution via profound analysis of specific overparameterized models.

\begin{figure*}[t]  
\centering
\subfigure[Linear Regression]{
\begin{minipage}{0.45\linewidth}
\centerline{\includegraphics[width=1\textwidth]{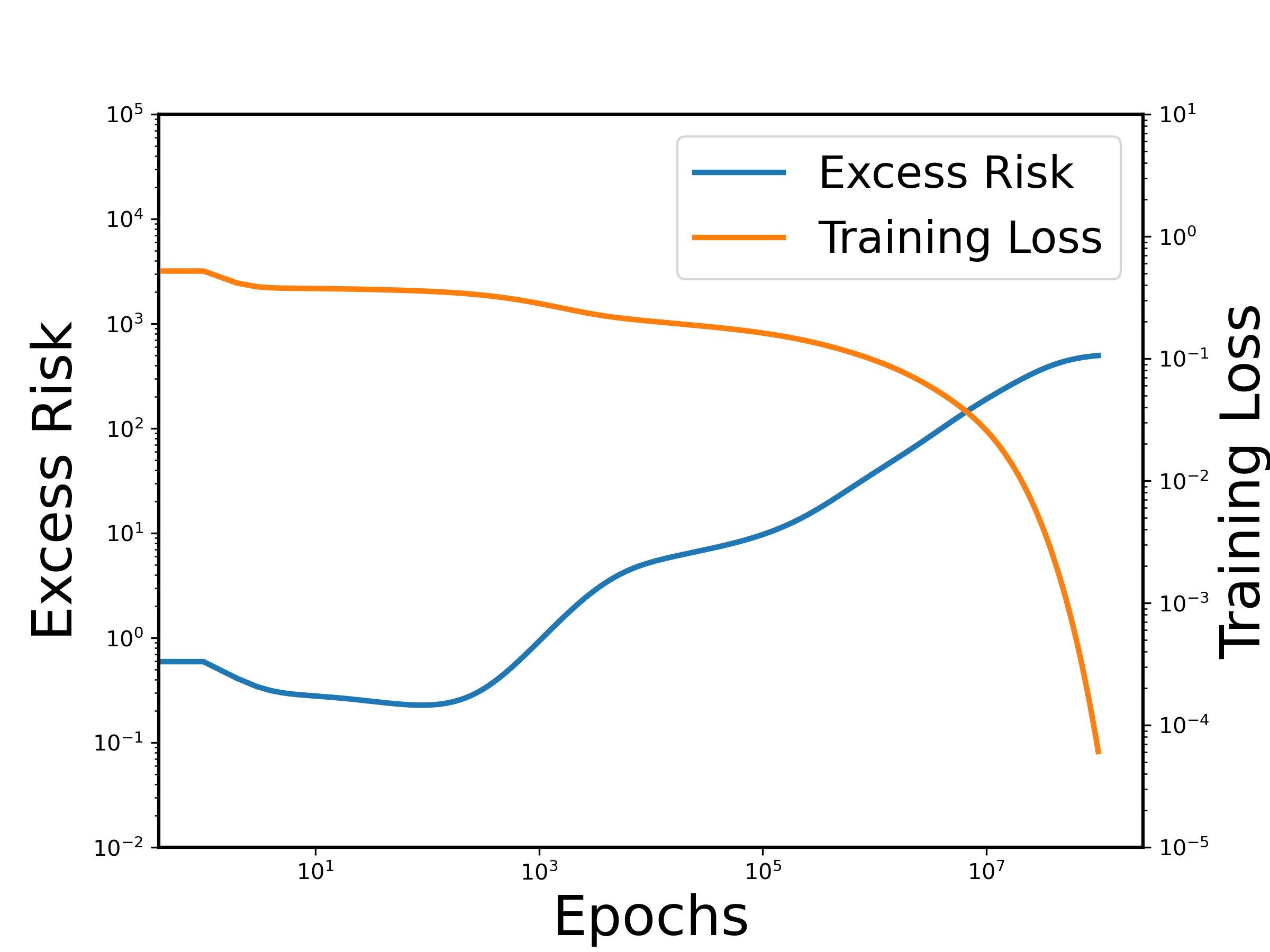}}
\end{minipage}
}
\quad
\subfigure[Corrupted MNIST]{
\begin{minipage}{0.45\linewidth}
\centerline{\includegraphics[width=1\textwidth]{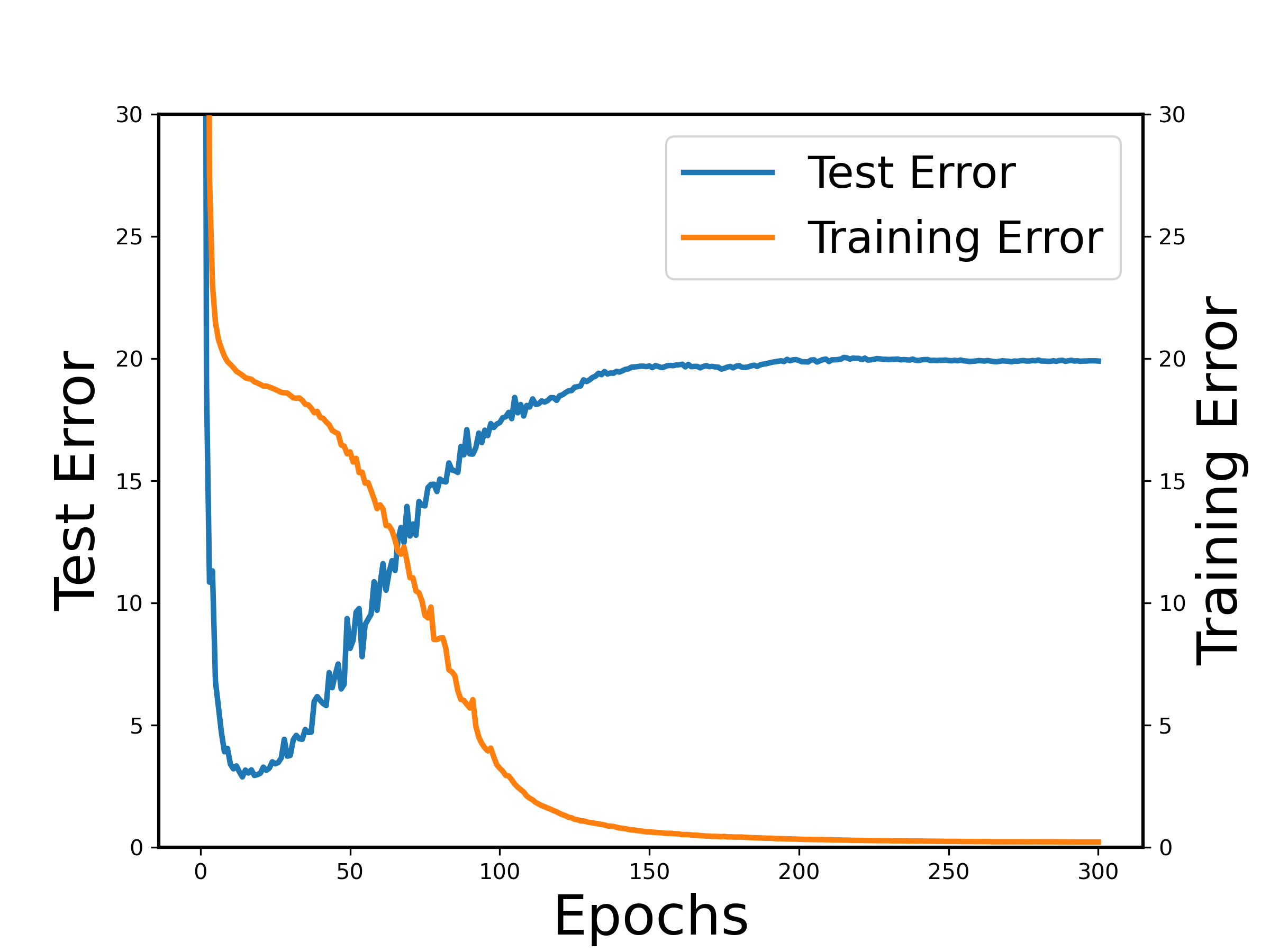}}
\end{minipage}}
\caption{\textbf{(a) The training plot for linear regression with spectrum $\lambda_i=1/i^2$ using GD.} Note that the axes are in the log scale. \textbf{(b) The training plot of CNN on corrupted MNIST with 20\% label noise using SGD.}  Both models successfully learn the useful features in the initial phase of training, but it takes a long time for them to fit the noise in the dataset. The observations demonstrate the power of 
data-dependent trajectory analysis, since the early stopping solutions on the trajectory generalize well but the final iterate fails to generalize.
See Appendix~\ref{section additional experiments} for details.}
\label{fig:intro}
\end{figure*}

Despite the significant role of data and algorithm in generalization analysis, existing theories usually focus on either the data factor~(\eg, uniform convergence~\citep{DBLP:conf/nips/NagarajanK19} and last iterate analysis~\citep{bartlett2020benign,tsigler2020benign}) or the algorithm factor~(\eg, stability-based bounds~\citep{DBLP:conf/icml/HardtRS16}).
Combining both data and algorithm factor into generalization analysis can help derive tighter generalization bounds and explain the generalization ability of overparameterized models observed in practice.
In this sense, a natural question arises:
\begin{center}
    \textit{How to incorporate both data factor and algorithm factor into generalization analysis?}
\end{center}

To gain insight into the interplay between data and algorithms, we provide motivating examples of a synthetic overparameterized linear regression task and a classification task on the corrupted MNIST dataset in figure \ref{fig:intro}. 
In both scenarios, the final iterate with less algorithmic information, which may include the algorithm type (\eg, GD or SGD), hyperparameters (\eg, learning rate, number of epochs), generalizes much worse than the early stopping solutions (see the Blue Line). 
In the linear regression case, the generalization error of the final iterate can be more than $\times 100$ larger than that of the early stopping solution. 
In the MNIST case, the final iterate on the SGD trajectory has 19.9\% test error, much higher than the 2.88\% test error of the best iterate on the GD trajectory. Therefore, the almost ubiquitous strategy of early stopping is a key ingredient in generalization analysis for overparameterized models, whose benefits have been demonstrated both theoretically and empirically \citep{yao2007early,ali2019continuous,DBLP:conf/aistats/LiSO20,advani2020high,ji2021early, DBLP:conf/nips/MallinarSAPBN22}. By focusing on the entire optimization trajectory and performing data-dependent trajectory analysis, both data information and the dynamics of the training algorithm can be exploited to yield consistent generalization bounds.

When we take the algorithm into consideration and analyze the data-dependent training trajectory, generalization occurs if the minimum excess risk of the iterates on the training trajectory converges to zero, as the sample size tends to infinity. This accords with the real practice of training deep neural networks, where one can pick up the best parameter on the training trajectory, by calculating its loss on a validation dataset.
We dub this notion of generalization as as \textit{data-algorithm-compatibility}, which is formally defined in Section~\ref{sec: comp_defi}.

The significance of compatibility comes in three folds.
Firstly, it incorporates both data and algorithm factors into generalization analysis, and is suitable for the overparameterization regime (see Definition~\ref{def: compatibility}).
Secondly, it serves as a minimal condition for generalization, without which one cannot expect to find a consistent solution via standard learning procedures. Consequently, compatibility holds with only mild assumptions and applies to a wide range of problem instances (see Theorem~\ref{thm:compat_main}). 
Thirdly, it captures the algorithmic significance of early stopping in generalization. By exploiting the algorithm information along the entire trajectory, we arrive at better generalization bounds than the last-iterate analysis~(see Table \ref{tab:comparison} and \ref{tab:k1andoptimal} for examples).

To theoretically validate compatibility, we study it under overparameterized linear regression setting. 
Analysis of the overparameterized linear regression is a reasonable starting point to study more complex models like deep neural networks~\cite{DBLP:conf/icml/EmamiSPRF20,DBLP:journals/corr/abs-2201-08082}, since many phenomena of the high dimensional non-linear model are also observed in the linear regime (\eg, Figure~\ref{fig:intro}). Furthermore, the neural tangent kernel (NTK) framework~\citep{DBLP:conf/nips/JacotHG18,DBLP:conf/icml/AroraDHLW19} demonstrates that very wide neural networks trained using gradient descent with appropriate random initialization can be approximated by kernel regression in a reproducing kernel Hilbert space, which rigorously establishes a close relationship between overparameterized linear regression and deep neural network training. 

Specifically, we investigate solving overparameterized linear regression using gradient descent with constant step size, and prove that under some mild regularity conditions, gradient descent is compatible with overparameterized linear regression if the effective dimensions of the feature covariance matrix are asymptotically bounded by the sample size. Compared with the last-iterate analysis~\citep{bartlett2020benign}, the main theorems in this paper require significantly weaker assumptions, which demonstrates the benefits of data-relevant and algorithm-relevant generalization analysis.

We summarize our contributions as follows:
\vspace{-0.5em}
\begin{itemize}
\item We formalize the notion of data-algorithm-compatibility, which highlights the interaction between data and algorithm and serves as a minimal condition for generalization.

\item We derive a sufficient condition for compatibility in solving overparameterized linear regression with gradient descent. 
Our theory shows that generalization of early-stopping iterates requires much weaker restrictions in the considered setting. 

\item Technically, we derive time-variant generalization bounds for overparameterized linear regression via data-dependent trajectory analysis. Empirically, we conduct the various experiments to verify the the theoretical results and demonstrate the benefits of early stopping.
\end{itemize}
\vspace{-0.7em}

%% file: text/relatedWork.tex
\section{Related Works}\label{sec: related works}

\textbf{Data-Dependent Techniques} mainly focus on the data distribution condition for generalization.
One of the most popular data-dependent methods is uniform convergence~\citep{koltchinskii2000rademacher,DBLP:conf/nips/BartlettFT17,DBLP:conf/nips/ZhouSS20,DBLP:journals/cacm/ZhangBHRV21}.
However, recent works~\citep{DBLP:conf/nips/NagarajanK19,DBLP:conf/icml/NegreaD020} point out that uniform convergence may not be powerful enough to explain generalization, because it may only yield inconsistent bound in even linear regression cases.
Another line of works investigates benign overfitting, which mainly studies generalization of overfitting solutions~\citep{bartlett2020benign,zou2021benign,tsigler2020benign,DBLP:journals/corr/abs-2008-02901,DBLP:conf/icassp/WangT21,DBLP:journals/corr/abs-2202-05928, DBLP:conf/nips/ZouWBGK22}. 

\textbf{Algorithm-Dependent Techniques} measure the role of the algorithmic information in generalization.
A line of works derives generalization bounds via algorithm stability~\citep{DBLP:conf/icml/HardtRS16,DBLP:conf/nips/FeldmanV18,DBLP:conf/colt/Mou0Z018,DBLP:conf/colt/FeldmanV19,DBLP:conf/colt/BousquetKZ20,DBLP:conf/iclr/LiLQ20,DBLP:conf/icml/LeiY20,DBLP:conf/nips/BassilyFGT20,DBLP:journals/corr/abs-2106-06153}.
A parallel line of works analyzes the implicit bias of algorithmic information~\citep{DBLP:journals/jmlr/BousquetE02,DBLP:conf/iclr/SoudryHNS18,DBLP:conf/nips/ShahTR0N20,DBLP:conf/nips/HuXAP20,DBLP:conf/iclr/LyuL20,DBLP:conf/nips/LyuLWA21}, which are mainly based on analyzing a specific data distribution (\eg, linear separable).

\textbf{Other Generalization Techniques.} 
Besides the techniques above, there are many other approaches. For example, PAC-Bayes theory performs well empirically and theoretically~\citep{DBLP:conf/colt/Shawe-TaylorW97,DBLP:journals/jmlr/Seeger02,DBLP:conf/colt/McAllester03,DBLP:journals/jmlr/Parrado-HernandezASS12,DBLP:journals/corr/McAllester13,DBLP:conf/uai/DziugaiteR17,DBLP:conf/iclr/NeyshaburBS18} and can yield non-vacuous bounds in deep learning regimes~\citep{DBLP:conf/nips/RivasplataKSS20,DBLP:journals/jmlr/Perez-OrtizRSS21}. 
Furthermore, there are other promising techniques including information theory~\citep{DBLP:conf/aistats/RussoZ16,DBLP:conf/nips/XuR17,DBLP:journals/corr/abs-2105-01747} and compression-based bounds~\citep{DBLP:conf/icml/Arora0NZ18}.

\textbf{Early Stopping} has the potential to improve generalization for various machine learning problems~\citep{DBLP:journals/jmlr/RaskuttiWY14,DBLP:conf/nips/VaskeviciusKR20,DBLP:journals/cacm/ZhangBHRV21,DBLP:conf/nips/LiNHW21,DBLP:conf/colt/KuzborskijS21,DBLP:journals/corr/abs-2106-15853,DBLP:journals/corr/abs-2202-09885, DBLP:journals/corr/abs-2306-02533}.
A line of works studies the rate of early stopping in linear regression and kernel regression with different algorithms, \emph{e.g.}, gradient descent~\citep{yao2007early}, stochastic gradient descent~\citep{DBLP:journals/tit/TarresY14,DBLP:conf/nips/RosascoV15,dieuleveut2016nonparametric,DBLP:journals/jmlr/LinR17,DBLP:conf/nips/Pillaud-VivienR18}, gradient flow~\citep{ali2019continuous}, conjugate gradient~\citep{blanchard2016convergence} and spectral algorithms~\citep{DBLP:journals/neco/GerfoROVV08,DBLP:journals/corr/abs-1801-06720}.
Beyond linear models, early-stopping is also effective for training deep neural networks~\citep{DBLP:conf/aistats/LiSO20,ji2021early}.
Another line of research focuses on the signal for early stopping \citep{DBLP:series/lncs/Prechelt12,DBLP:conf/pkdd/ForouzeshT21}.

%% file: text/preliminary.tex
\newcommand{\sample}{{\boldsymbol{z}}}
\newcommand{\para}{{\boldsymbol{\theta}}}

\newcommand{\boldZ}{{\boldsymbol{Z}}}

\section{Preliminaries}
\label{sec: compatibility}
In this section, we formally define compatibility between the data distribution and the training algorithm, starting from the basic notations.

\subsection{Notations}

\textbf{Data Distribution.}
Let $\cD$ denote the population distribution and $\sample \sim \cD$ denote a data point sampled from distribution $\cD$.
Usually, $\sample$ contains a feature and its corresponding response. 
Besides, we denote the dataset with $n$ samples as $\boldZ \triangleq \{\sample_i\}_{i \in [n]}$, where $\sample_i\sim \cD$ are i.i.d.~sampled from distribution $\cD$.

\textbf{Loss and Excess Risk.}
Let $\ell(\para;\sample)$ denote the loss on sample $\sample$ with parameter $\para \in \bR^{p}$.
The corresponding population loss is defined as $L(\para; \cD) \triangleq \bE_{\sample \sim \cD} \ell(\para;\sample)$.
When the context is clear, we omit the dependency on $\cD$ and denote the population loss by $L(\para)$.
Our goal is to find the optimal parameter $\para^*$ which minimizes the population loss, i.e., $L(\para^*) = \min_\para L(\para)$.
Measuring how a parameter $\para$ approaches $\para^*$ relies on a term \emph{excess risk} $R(\para)$, defined as $R(\para) \triangleq L(\para) - L(\para^*) $.

\textbf{Algorithm.}
Let $\cA(\cdot)$ denote a iterative algorithm that takes training data $\mb Z$ as input and outputs a sequence of parameters $\{\para_n^{(t)}\}_{t\ge 0}$, where $t$ is the iteration number.
The algorithm can be either deterministic or stochastic, \eg, variants of (S)GD.

\vspace{-0.3em}

\subsection{Definitions of Compatibility}\label{sec: comp_defi}
Based on the above notations, we introduce the notion of compatibility between data distribution and algorithm in Definition~\ref{def: compatibility}.
Informally, compatibility measures whether a consistent excess risk can be reached along the training trajectory.
Note that we omit the role of the loss function in the definition, although the algorithm depends on the loss function.

\begin{definition}[\textbf{Compatibility}]
\label{def: compatibility}
Given a loss function $\ell(\cdot)$ with corresponding excess risk $R(\cdot)$, a data distribution $\cD$ is compatible with an algorithm $\cA$ if there exists nonempty subsets $T_n$ of $\Nbb$, such that  $\sup_{t\in T_n} R(\mb \theta_n^{(t)}) $ converges to zero in probability as sample size $n$ tends to infinity, 

where $\{\mb \theta_{n}^{(t)}\}_{t \geq 0}$ denotes the output of algorithm $\cA$, and the randomness comes from the sampling of training data $\mb Z$ from distribution $\cD$ and the execution of algorithm $\cA$.
That is to say, $(\cD, \cA)$ is compatible if there exists nonempty sets $T_n$, such that 
\begin{equation}
\label{eqn: compatible}
    \sup_{t\in T_n} R(\mb \theta_n^{(t)}) \overset{P}{\to} 0 \text{\quad as \quad } n \to \infty.
\end{equation}
We call $\{T_n\}_{n>0}$ the compatibility region of $(\cD,\cA)$. The distribution $\mathcal{D}$ is allowed to change with $n$. In this case, $\mathcal{D}$ should be understood as a sequence of distributions $\{\mathcal{D}_n\}_{n\ge 1}$. We also allow the dimension of model parameter $\mb \theta$ to be infinity or to grow with $n$. We omit this dependency on $n$ when the context is clear.
\end{definition}

Compatibility serves as a minimal condition for generalization, since if a data distribution is incompatible with the algorithm, one cannot expect to reach a small excess risk even if we allow for \emph{arbitrary} early stopping. However, we remark that considering only the minimal excess risk is insufficient for a practical purpose, as one cannot exactly find the $t$ that minimizes $R(\mb \theta_n^{(t)})$ due to the noise in the validation set. Therefore, it is meaningful to consider a region of time $t$ on which the excess risk is consistent as in Definition~\ref{def: compatibility}. The larger the region is, the more robust the algorithm will be to the noise in its execution.

\textbf{Comparisons with Other Notions.} Compared to classic definitions of learnability, \eg, PAC learning, the definition of compatibility is data-specific and algorithm-specific, and is thus a more fine-grained notion. 
Compared to the concept of \emph{benign} proposed in~\citep{bartlett2020benign}, which studies whether the excess risk at $t=\infty$ converges to zero in probability as the sample size goes to infinity, 
compatibility only requires that there exists a time to derive a consistent excess risk.
We will show later in Section~\ref{sec: compat_thm} that in the overpamameterized linear regression setting, there exist cases such that the problem instance is compatible but not benign.

\section{Analysis of Overparameterized Linear Regression with Gradient Descent}
\label{sec: comp for linear reg}
To validate the meaningfulness of compatibility, 
we study it in the overparameterized linear regression regime.
We first introduce the data distribution, loss, and training algorithm, and then present the main theorem, which provides a sufficient condition for compatibility in this setting.

\subsection{Preliminaries for Overparameterized Linear Regression}\label{sec: linreg_def}
\textbf{Notations.}  Let $O,o,\Omega,\omega$ denote asymptotic notations, with their usual meaning. For example, the argument~$a_n=O(b_n)$ means that there exists a large enough constant $C$, such that $a_n\le C b_n$. We use~$\lesssim$ with the same meaning as the asymptotic notation $O$.
Besides, let $\|\mb x\|$ denote the $\ell_2$ norm for vector $\mb x$, and $\|\mb A\|$ denote the operator norm for matrix $\mb A$.
We allow the vector to belong to a countably infinite-dimensional Hilbert space $\mathcal{H}$, and  with a slight abuse of notation, we use $\mathbb{R}^\infty$ interchangeably with $\mathcal{H}$. In this case, $x^\top z$ denotes inner product and $xz^\top$ denotes tensor product for $x, z\in \mathcal{H}$. A random variable $X$ is called $\sigma$-subgaussian if $\ee[e^{\lambda X}]\le e^{\lambda^2\sigma^2/2}$ for any $\lambda$.

\textbf{Data Distribution.}
Let $(\mb{x} ,y) \in \mathbb{R}^p\times\mathbb{R}$ denote the feature vector and the response, following a joint distribution $\cD$.
Let $\mb \Sigma \triangleq \ee[\mb x \mb x^\top]$ denote the feature covariance matrix, whose eigenvalue decomposition is $\mb \Sigma=\mb V\mb \Lambda \mb V^\top=\sum_{i>0}\lambda_i \mb v_i\mb v_i^\top$ with decreasing eigenvalues $\lambda_1\ge\lambda_2\ge\cdots$.
We make the following assumptions on the distribution of the feature vector.
\begin{assumption}[Assumptions on feature distribution]\label{assumption 1}
We assume that
\begin{enumerate}
    \item $\ee[\mb x]=0$.
    \item $\lambda_1>0, \sum_{i>0}\lambda_i<C$ for some absolute constant $C$.
    \item Let $\Tilde{\mb x}=\mb \Lambda^{-\frac{1}{2}}\mb V^\top \mb x$. The random vector $\Tilde{\mb x}$ has independent $ \sigma_x$-subgaussian entries.
\end{enumerate}
\end{assumption}

\textbf{Loss and Excess Risk.}
We choose square loss as the loss function $\ell$, 
i.e. $\ell(\para,(\mb x,y))= 1/2 (y-\mb x^\top \para)^2$.
The corresponding population loss is denoted by $L(\para) = \bE\ell(\para,(\mb x,y))$
and the optimal parameter is denoted by $\para^* \triangleq \argmin_{\para \in \mathbb{R}^p} L(\para)$. 
We assume that $\|\mb \theta^*\|<C$ for some absolute constant $C$.
If there are multiple such minimizers, we choose an arbitrary one and fix it thereafter. 
We focus on the excess risk of parameter $\para$, defined as
\begin{equation}
    R(\mb \theta) = L(\para) - L(\para^*) = \frac{1}{2}(\para - \para^*)^\top \mb \Sigma (\para - \para^*).
\end{equation}

Let $\varepsilon=y-\mb x^\top \para^*$ denote the noise in data point $(\mb x,y)$.
The following assumptions involve the conditional distribution of the noise.
\begin{assumption}[Assumptions on noise distribution]\label{assumption 2}
We assume that 
\begin{enumerate}
    \item The conditional noise $\varepsilon|\mb x$ has zero mean.
    \item The conditional noise $\varepsilon|\mb x$ is $\sigma_y$-subgaussian.
\end{enumerate}
\end{assumption}
Note that both Assumption~\ref{assumption 1} and Assumption~\ref{assumption 2} are commonly considered in the related literatures~\citep{bartlett2020benign,tsigler2020benign,zou2021benign}.

\textbf{Training Set.}
Given a training set $\{(\mb x_i,y_i)\}_{1\le i \le n}$ with $n$ pairs independently sampled from the population distribution $\cD$, 
we define $\mb X \triangleq (\mb x_1,\cdots, \mb x_n)^\top\in\mathbb{R}^{n\times p}$ as the feature matrix, 
$\mb{Y} \triangleq (y_1,\cdots, y_n)^\top\in\mathbb{R}^n$ as the corresponding noise vector, and $\mb {\varepsilon}  \triangleq \mb Y-\mb X \para^*$ as the residual vector.
Let the singular value decomposition (SVD) of $\mb X$ be $\mb{X}=\mb{U}\mb{\Tilde{\Lambda}}^{\frac{1}{2}}\mb{W}^\top$, with $\mb{ \Tilde{\Lambda}}=\text{diag}\{\mu_1\,\cdots,\mu_n\} \in \mathbb{R}^{n\times n}$, $\mu_1\ge \cdots\ge \mu_n$. 

We consider the overparameterized regime where the feature dimension is larger than the sample size, namely, $p> n$. 
In this regime, 
we assume that $\text{rank}(\mb X)=n$ almost surely as in \citet{bartlett2020benign}. This assumption is equivalent to the invertibility of $XX^\top$.
\begin{assumption}[Linear independent training set]
\label{assump:fullRank}
For any $n < p$, we assume that the features in the training set $\{ \mb x_1, \mb x_2, \cdots, \mb x_n \}$ is linearly independent almost surely.
\end{assumption}

\textbf{Algorithm.}
Given the dataset $(\mb{X}, \mb Y)$, define the empirical loss function as $\hat{L}(\para) \triangleq \frac{1}{2n}\|\mb X\para-\mb Y\|^2.$
We choose full-batch gradient descent on the empirical risk with a constant learning rate $\lambda$ as the algorithm $\cA$ in the previous template. 
In this case, the update rule for the optimization trajectory $\{\para_t\}_{t\ge 0}$ is formulated as
\begin{equation}\label{eq: gd dynamic}
\para_{t+1}=\para_t-\frac{\lambda}{n}\mb X^\top (\mb X\para_t-\mb Y).
\end{equation}

Without loss of generality, we consider zero initialization $\para_0 = \mb 0$ in this paper. 
In this case, for a sufficiently small learning rate $\lambda$, $\para_t$ converges to the \emph{min-norm interpolator} $ \hat{\para}=\mb X^\top(\mb X\mb X^\top)^{-1} \mb Y$ as $t$ goes to infinity, which was well studied previously~\citep{bartlett2020benign}. This paper takes one step further and discuss the excess risk along the entire training trajectory $\{R(\para_t)\}_{t\ge 0}$.

\textbf{Effective Rank and Effective Dimensions.}
We define the effective ranks of the feature matrix $\mb \Sigma$ as 
$
r(\mb\Sigma) \triangleq \frac{\sum_{i>0}\lambda_i}{\lambda_1}
$
and
$R_{k}(\mb \Sigma) \triangleq \frac{(\sum_{i>k}\lambda_i)^2}{\sum_{i>k}\lambda_i^2}$.
Our results depend on two notions of effective dimension of the feature covariance $\mb \Sigma$, defined as 
\begin{align}\label{def: k0}
    &k_0 \triangleq \min\left \{l\ge 0:\lambda_{l+1}\le \frac{c_0\sum_{i>l}\lambda_i}{n}\right\},\\
    &k_1 \triangleq  \min \left \{l\ge 0:\lambda_{l+1}\le \frac{c_1\sum_{i>0}\lambda_i}{n}\right\},
\end{align}

where $c_0,c_1$ are constants independent of the dimension $p$, sample size $n$, and time $t$\footnote{Constants may depend on $\sigma_x$, and we omit this dependency thereafter for clarity.}. We omit the dependency of $k_0,k_1$ on $c_0,c_1,n,\mb\Sigma$ when the context is clear.

\subsection{Main Theorem for Overparameterized Linear Regression with Gradient Descent }\label{sec: compat_thm}
Next, we present the main result of this section, which provides a clean condition for compatibility between gradient descent and overparameterized linear regression.

\begin{theorem}\label{thm:compat_main} 
Consider the overparameterized linear regression setting defined in section \ref{sec: linreg_def}. 
Let Assumption \ref{assumption 1},\ref{assumption 2} and \ref{assump:fullRank} hold. 
Assume the learning rate satisfies $\lambda=O\left(\frac{1}{\Tr(\mb \Sigma)}\right)$.
\begin{itemize}
    \item If the covariance satisfies $k_0=o(n),R_{k_0}(\mb \Sigma)=\omega(n), \ r(\mb \Sigma)=o(n)$, it is compatible in the region $T_n=\left(\omega\left(\frac{1}{\lambda}\right), \infty\right)$.
    \item If the covariance satisfies $k_0=O(n),k_1=o(n),r(\mb \Sigma)=o(n)$, it is compatible in the region $T_n=\left(\omega\left(\frac{1}{\lambda}\right),o\left(\frac{n}{\lambda}\right)\right)$.
    \item If the covariance does not change with $n$, and satisfies $k_0=O(n)$ and $p=\infty$, it is compatible in the region $T_n=\left(\omega\left(\frac{1}{\lambda}\right),o\left(\frac{n}{\lambda}\right)\right)$.
\end{itemize}
\end{theorem}

The proof of Theorem \ref{thm:compat_main} is given in Appendix \ref{appendix:proof} and sketched in Section~\ref{sec: proof sketch}. The theorem shows that 
gradient descent is compatible with overparameterized linear regression under some mild regularity conditions
on the learning rate, effective rank and effective dimensions.
The condition on the learning rate is natural for optimizing a smooth objective. 
We conjecture that the condition $k_0=O(n)$ can not be removed in general cases, since the effective dimension $k_0$ characterizes the concentration of the singular values of the data matrix $\mb X$ and plays a crucial role in the excess risk of the gradient descent dynamics.

\textbf{Comparison with Benign Overfitting.} The paper \cite{bartlett2020benign} studies overparameterized linear regression and gives the condition for min-norm interpolator to generalize.
They prove that the feature covariance $\mb \Sigma$ is benign if and only if 
\begin{equation}
    k_0=o(n),\ R_{k_0}(\mb \Sigma) =\omega(n), \ r(\mb \Sigma)=o(n)
\end{equation}

As discussed in Section \ref{sec: comp_defi}, benign problem instance also satisfies compatibility, since benign overfitting requires a stronger condition on $k_0$ and an additional assumption on $R_{k_0}(\mb \Sigma)$. The following example shows that this inclusion relationship is strict.
\begin{example}\label{exa: benign_fail}
Under the same assumption as in Theorem~\ref{thm:compat_main},
if the spectrum of $\mb \Sigma$ satisfies 
\begin{equation}
    \lambda_k=\frac{1}{k^\alpha},
\end{equation}
for some $\alpha>1$, we derive that $k_0=\Theta(n)$. Therefore, this problem instance satisfies compatibility, but does not satisfy benign overfitting.
\end{example}

Example \ref{exa: benign_fail} shows the existence of a case where the early stopping solution can generalize but interpolating solution fails. Therefore, Theorem~\ref{thm:compat_main} can characterize generalization for a much wider range of problem instances.

%% file: text/mainResults.tex
\section{Proof Sketch and Techniques}
\label{sec: proof sketch}

\subsection{A Time Variant Bound}

We further introduce an additional type of effective dimension besides $k_0,k_1$, which is time variant and is utilized to track the optimization dynamics.
\begin{definition}[Effective Dimensions]
Given a feature covariance matrix $\mb \Sigma$, define the effective dimension $k_2$ as
\begin{align}\label{effective dimensions}\begin{aligned}
    &k_2 \triangleq \min\left\{l\ge 0:\sum_{i>l}\lambda_i+n\lambda_{l+1}\le c_2 c(t,n) \sum_{i>0}\lambda_i\right\},
\end{aligned}\end{align}
where $c_2$ is a constant independent of the dimension $p$, sample size $n$, and time $t$.
The term $c(t, n)$ is a function to be discussed later.
When the context is clear, we omit its dependencies on $c_2, c(t, n), n, \mb\Sigma$ and only denote it by $k_2$.
\end{definition}

Based on the effective rank and effective dimensions defined above, we provide a time-variant bound in Theorem~\ref{main thm} for overparameterized linear regression, which further leads to the argument in Theorem~\ref{thm:compat_main}.
Compared to the existing bound~\citep{bartlett2020benign}, Theorem~\ref{main thm} focuses on investigating the role of training epoch $t$ in the excess risk, to give a refined bias-variance decomposition. 

\begin{theorem}[Time Variant Bound]\label{main thm}
Suppose Assumption \ref{assumption 1}, \ref{assumption 2} and \ref{assump:fullRank} hold.
Fix a function $c(t, n)$.
Given $\delta\in(0,1)$, assume that $k_0\le \frac{n}{c}$, $\log \frac{1}{\delta}\le \frac{n}{c}$, $0< \lambda\le \frac{1}{c\sum_{i>0}\lambda_i}$ for a large enough constant $c$. Then with probability at least $1-\delta$, we have for any $t\in \Nbb$,
\begin{equation}
     R(\mb\theta_t)\lesssim B(\mb\theta_t)+V(\mb\theta_t),
\end{equation}
where

\begin{small}
\begin{equation*}
    \begin{split}
        &B(\mb\theta_t)=\left\|\mb\theta^*\right\|^2\left(\frac{1}{\lambda t}+
\|\mb\Sigma\|
\max\left\{\sqrt{\frac{r(\mb\Sigma)}{n}},
\frac{r(\mb\Sigma)}{n},\sqrt{\frac{\log(\frac{1}{\delta})}{n}}\right\}\right),\\
&V(\mb\theta_t)
=\sigma_y^2\log\left(\frac{1}{\delta}\right)
\left(\frac{k_1}{n}+\frac{k_2}{c(t,n)n}+
c(t,n)\left(\frac{\lambda t}{n}\sum_{i>0}\lambda_i\right)^2\right).
    \end{split}
\end{equation*}
\end{small}

\end{theorem}

 We provide a high-level intuition behind Theorem~\ref{main thm}.
We decompose $R(\mb\theta_t)$ into the bias term and the variance term. The variance term is then split into the leading part and tailing part based on the sprctrum of the feature covariance $\mb\Sigma$.
The eigenvalues in the tailing part will cause the variance term in the excess risk of the min-norm interpolating solution to be $\Omega(1)$ for fast decaying spectrum, as is the case in~\citep{bartlett2020benign}. However, since the convergence in the tailing eigenspace is slower compared with the leading eigenspace, a proper early stopping strategy will prevent the overfitting in the tailing eigenspace and meanwhile avoid underfitting in the leading eigenspace.

\textbf{The $c(t, n)$ Principle.}
It is worth emphasizing that our bound holds for arbitrary positive function $c(t,n)$. 
Therefore, one can fine-tune the generalization bound by choosing a proper $c(t, n)$.
In the subsequent sections, we show how to derive consistent risk bounds for different time $t$, based on different choices of $c(t,n)$. We present the case of choosing a constant $c(t,n)$ in the next section. We leave the case of choosing a varying $c(t,n)$ to Appendix~\ref{further examples}.

\subsection{Varying \texorpdfstring{$t$}{Lg}, Constant \texorpdfstring{$c(t,n)$}{Lg}}\label{sec:proof_compat}

\begin{table*}[t]
\caption{\textbf{Comparisons of excess risk bound with \citet{bartlett2020benign} and \citet{zou2021benign}.}
We provide four types of feature covariance with eigenvalues $\lambda_k$, including \emph{Inverse Polynomial} ($\lambda_k=\frac{1}{k^\alpha}$, $\alpha>1$), \emph{Inverse Log Polynomial} ($\lambda_k=\frac{1}{k \log^\beta (k+1)}$, $\beta>1$), \emph{Constant} ($\lambda_k=\frac{1}{n^{1+\varepsilon}}$, $1\le k \le n^{1+\varepsilon}$, $\varepsilon>0$), and \emph{Piecewise Constant} ($\lambda_k = \frac{1}{s}$ if $1\leq k \leq s$ and $\lambda_k = \frac{1}{d-s}$ if $s+1 \leq k \leq d$, where $s= n^r,d=n^q,0<r\le 1,q> 1$). 
In light of these bounds, ours outperforms \citet{bartlett2020benign} in all the cases, and outperforms \citet{zou2021benign} in Constant / Piecewise Constant cases if $\varepsilon<\frac{1}{2}$ and $q<\min\{2-r,\frac{3}{2}\}$. We refer to Appendix~\ref{appen:exa} for more details.}
\label{tab:comparison}
\vskip 0.15in
\begin{center}
\begin{small}
\begin{sc}
\setlength{\tabcolsep}{0.15mm}{
\begin{tabular}{lcccr}
\toprule
Distributions & Ours & \citet{bartlett2020benign} & \citet{zou2021benign} \\
\midrule
\midrule
Inverse polynomial& $O\left(n^{- \min \left\{\frac{\alpha-1}{\alpha},\frac{1}{2}\right\}}\right)$& $O(1)$ & $O\left(n^{-\frac{\alpha-1}{\alpha}}\right)$\\
\midrule
Inverse log polynomial& $O\left(\frac{1}{\log^\beta n}\right)$& $o(1)$ & $O\left(\frac{1}{\log^\beta n}\right)$\\
\midrule
Constant & $O\left(n^{-\frac{1}{2}}\right)$ &$O\left(n^{- \min \left\{\varepsilon,\frac{1}{2}\right\}}\right)$ &$O\left(n^{-\min\{\varepsilon,1\}}\right)$ \\
\midrule
Piecewise constant & $O\left(n^{- \min \left\{1-r,\frac{1}{2}\right\}}\right)$ &$O\left(n^{-\min
\left\{1-r,q-1,\frac{1}{2} \right\}}\right)$ &$O\left(n^{-\min\left\{1-r,q-1\right\}}\right)$\\
\bottomrule
\end{tabular}
}
\end{sc}
\end{small}
\end{center}
\end{table*}

Theorem~\ref{main thm} provides an excess risk upper bound uniformly for $t\in\Nbb$.
However, it is still non-trivial to derive Theorem~\ref{thm:compat_main}, where the remaining question is to decide the term $c(t, n)$.
The following corollary shows the generalization bound when setting $c(t, n)$ to a constant.

\begin{corollary}\label{sqrt n}
Let Assumption \ref{assumption 1}, \ref{assumption 2} and \ref{assump:fullRank} hold. Fix a constant $c(t, n)$. Suppose $k_0= O(n)$, $k_1=o(n)$, $r(\mb\Sigma)=o(n)$, $\lambda=O\left(\frac{1}{\sum_{i>0}\lambda_i}\right)$.
Then there exists a sequence of positive constants $\{\delta_n\}_{n\ge 0}$ which converge to 0, such that with probability at least $1-\delta_n$, the excess risk is consistent for $t\in\left(\omega\left(\frac{1}{\lambda}\right),o\left(\frac{n}{\lambda}\right)\right)$, i.e. 
\begin{equation}
    R(\mb\theta_t)=o(1).
\end{equation}
Furthermore, for any positive constant $\delta$, with probability at least $1-\delta$, the minimal excess risk on the training trajectory can be bounded as
\begin{equation}
    \min_t R(\mb\theta_t)\lesssim \frac{\max\{\sqrt{r(\mb\Sigma)},1\}}{\sqrt{n}}
+\frac{\max\{k_1,1\}}{n}.
\end{equation}
\end{corollary}

Lemma~\ref{lem:k1} below shows that $k_1=o(n)$ always holds for fixed distribution.
Therefore, combining Corollary~\ref{sqrt n} and the following Lemma~\ref{lem:k1} completes the proof of Theorem \ref{thm:compat_main}.

\begin{lemma}\label{lem:k1}
For any fixed (i.e. independent of sample size $n$) feature covariance $\mb \Sigma$ satisfying assumption \ref{assumption 1}, we have $k_1(n)=o(n)$.
\end{lemma}

Next we apply the bound in Corollary~\ref{sqrt n} to several data distributions.
These distributions are widely discussed in \citep{bartlett2020benign,zou2021benign}. 
We also derive the existing excess risk bounds, which focus on the min-norm interpolator \citep{bartlett2020benign} and one-pass SGD iterate \citep{zou2021benign}, of these distributions and compare them with our theorem. The results are summarized in Table \ref{tab:comparison}, which shows that the bound in Corollary~\ref{sqrt n} outperforms previous results for a general class of distributions.

\begin{example}\label{Allexample}
Under the same conditions as Theorem~\ref{main thm}, let $\mb \Sigma$ denote the feature covariance matrix.
We show the following examples:
\begin{enumerate}
    \item \textbf{\emph{(Inverse Polynominal).}} If the spectrum of $\mb \Sigma$ satisfies $\lambda_k=\frac{1}{k^\alpha}$ for some $\alpha>1$, we derive that $k_0=\Theta(n)$, $k_1=\Theta\left(n^{\frac{1}{\alpha}}\right)$. Therefore, $\min_t V(\mb\theta_t)=O\left(n^{\frac{1-\alpha}{\alpha}}\right)$ and $\min_t R(\mb\theta_t)=O\left(n^{- \min \left\{\frac{\alpha-1}{\alpha},\frac{1}{2}\right\}}\right)$.
    
    \item \textbf{\emph{(Inverse Log-Polynominal).}} If the spectrum of $\mb \Sigma$ satisfies $\lambda_k=\frac{1}{k \log^\beta (k+1)}$ for some $\beta>1$, we derive that $k_0=\Theta\left(\frac{n}{\log n}\right)$, $k_1=\Theta\left(\frac{n}{\log^\beta n}\right)$.
    Therefore, $\min_t V(\mb\theta_t)=O\left(\frac{1}{\log^\beta n}\right)$ and $\min_t R(\mb\theta_t)=O\left(\frac{1}{\log^\beta n}\right)$.
    
    \item \textbf{\emph{(Constant).}} If the spectrum of $\mb \Sigma$ satisfies $\lambda_k=\frac{1}{n^{1+\varepsilon}},1\le k \le n^{1+\varepsilon}$,
    for some $\varepsilon>0$, we derive that $k_0=0$, $k_1=0$. Therefore, $\min_t V(\mb\theta_t)=O\left(\frac{1}{n}\right)$ and $\min_t R(\mb\theta_t)=O\left(\frac{1}{\sqrt{n}}\right)$.
    
    \item \textbf{\emph{(Piecewise Constant).}} If the spectrum of $\mb \Sigma$ satisfies $\lambda_k=\begin{cases}\frac{1}{s}&1\le k\le s,\\ \frac{1}{d-s}& s+1\le k\le d,\end{cases}$ where $s= n^r,d=n^q,0<r\le 1,q\ge 1$. We derive that $k_0=n^r$, $k_1=n^r$. Therefore, $\min_t V(\mb\theta_t)=O(n^{r-1})$ and $\min_t R(\mb\theta_t)=O\left(n^{- \min \left\{1-r,\frac{1}{2}\right\}}\right).$
\end{enumerate}
\end{example}

\subsection{Varying \texorpdfstring{$t$}{Lg}, Varying \texorpdfstring{$c(t,n)$}{Lg}}\label{further examples}
Although setting $c(t, n)$ to a constant as in Corollary~\ref{sqrt n} suffices to prove Theorem~\ref{thm:compat_main}, in this section we show that the choice of $c(t, n)$ can be much more flexible.
Specifically, we provide a concrete example and demonstrate that by setting $c(t, n)$ to a non-constant, Theorem~\ref{main thm} can indeed produce larger compatibility regions.

\begin{example}
\label{eg: non-constant Example}
Under the same conditions as Theorem~\ref{main thm}, let $\mb \Sigma$ denote the feature covariance matrix.
If the spectrum of $\mb \Sigma$ satisfies $\lambda_k=\frac{1}{k^\alpha}$ for some $\alpha>1$, we set $c(t,n)=\Theta\left(n^{\frac{\alpha+1-2\alpha\tau}{2\alpha+1}}\right)$ for a given $\frac{\alpha+1}{2\alpha}\le\tau\le \frac{3\alpha+1}{2\alpha+2}$. 
Then for $t = \Theta(n^\tau)$, we derive that $V(\mb\theta_t)=O\left(n^{\frac{2\alpha\tau-3\alpha+2\tau-1}{2\alpha+1}}\right)$.

\end{example}

Example~\ref{eg: non-constant Example} shows that by choosing $c(t, n)$ as a non-constant, we exploit the full power of Theorem~\ref{main thm}, and extend the region to $t=\Theta\left(n^{\frac{3\alpha+1}{2\alpha+2}}\right)=\omega(n)$. In this example, Theorem~\ref{main thm} outperforms all $O\left(\frac{t}{n}\right)$-type bounds, which become vacuous when $t=\omega(n)$.

%% file: text/experiments.tex
\section{Experiments}\label{sec: exp}

\begin{table*}[t]
\caption{\textbf{The effective dimension $k_1$, the optimal early stopping excess risk, and the min-norm excess risk for different feature distributions, with sample size $n=100$ , $p= 1000$.}
  The table shows that early stopping solutions generalize significantly better than min-norm interpolators, and reveals a positive correlation between the effective dimension $k_1$ and excess risk of early stopping solution. We calculate the 95\% confidence interval for each excess risk.}
\vspace{-1em}
\label{tab:k1andoptimal}
\vskip 0.15in
\begin{center}
\begin{small}
\begin{sc}
\begin{tabular}{lccc}
\toprule
    \tabincell{c}{Distributions}  & $k_1$ &  \tabincell{c}{Optimal Excess Risk} & \tabincell{c}{Min-norm Excess Risk}\\
\midrule
$\lambda_i = \frac{1}{i}$& $\Theta\left(n\right)$& $2.399\pm 0.0061$  & $24.071\pm 0.2447$\\
$\lambda_i = \frac{1}{i^2}$ &$\Theta\left(n^{\frac{1}{2}}\right)$ & $0.214\pm 0.0050$ & $43.472\pm 0.6463$\\
$\lambda_i = \frac{1}{i^3}$ &$\Theta\left(n^{\frac{1}{3}}\right)$ & $0.077\pm 0.0005$ & $10.401\pm 0.2973$ \\
$\lambda_i = \frac{1}{i\log(i+1)}$& $\Theta\left(\frac{n}{\log n}\right)$& $0.697\pm 0.0053$&
$89.922\pm 0.9591$\\
$\lambda_i = \frac{1}{i\log^2(i+1)}$&$\Theta\left(\frac{n}{\log^2 n}\right)$ & $0.298\pm 0.0054$ & $82.413\pm 0.9270$\\
$\lambda_i = \frac{1}{i\log^3(i+1)}$&$\Theta\left(\frac{n}{\log^3 n}\right)$ & $0.187\pm 0.0047$ & $38.145\pm 0.5862$\\
\bottomrule
\end{tabular}
\end{sc}
\end{small}
\end{center}
\end{table*}
\vspace{-0.5em}
In this section, we provide numerical studies of overparameterized linear regression problems.
We consider overparameterized linear regression instances with input dimension $p=1000$, sample size $n=100$. The features are sampled from Gaussian distribution with different covariances. 
The empirical results
\textbf{(a.) demonstrate the benefits of trajectory analysis underlying the definition of compatibility,} since the optimal excess risk along the algorithm trajectory is significantly lower than that of the min-norm interpolator
\textbf{(b.) validate the statements in Corollary~\ref{sqrt n}}, since the optimal excess risk is lower when the eigenvalues of feature covariance decay faster.
We refer to Appendix~\ref{section additional experiments} for detailed setups, additional results and discussions.

\textbf{Observation One: Early stopping solution along the training trajectory generalizes significantly better than the min-norm interpolator.} 
We calculate the excess risk of optimal early stopping solutions and min-norm interpolators from 1000 independent trials and list the results in Table~\ref{tab:k1andoptimal}. 
The results illustrate that the early stopping solution on the algorithm trajectory enjoys much better generalization properties.
This observation corroborates the importance of data-dependent training trajectory in generalization analysis.

\textbf{Observation Two: The faster covariance spectrum decays, the lower optimal excess risk is.}
Table \ref{tab:k1andoptimal} also illustrates a positive correlation between the decaying rate of $\lambda_i$ and the generalization performance of the early stopping solution.
This accords with Theorem~\ref{main thm}, showing that the excess risk is better for a smaller effective dimension $k_1$, where small $k_1$ indicates a faster-decaying eigenvalue $\lambda_i$.
We additionally note that such a phenomenon also illustrates the difference between min-norm and early stopping solutions in linear regression, since \citet{bartlett2020benign} demonstrate that the min-norm solution is not consistent when the eigenvalues decay too fast. By comparison,  early stopping solutions do not suffer from this restriction.

%% file: text/conclusion.tex
\section{Conclusion}
In this paper, we investigate how to characterize and analyze generalization in a data-dependent and algorithm-dependent manner. We formalize the notion of data-algorithm compatibility and study it under the regime of overparameterized linear regression with gradient descent. Our theoretical and empirical results demonstrate that one can ease the assumptions and broaden the scope of generalization by fully exploiting the data information and the algorithm information.
Despite linear cases in this paper, compatibility can be a much more general concept. Therefore, we believe this paper will motivate more work on data-dependent trajectory analysis.

%% file: text/ack.tex
\section*{Acknowledgement}

The authors would like to acknowledge the support from the 2030 Innovation Megaprojects of China (Programme on New Generation Artificial Intelligence) under Grant No. 2021AAA0150000.

%% file: text/appendix_proof.tex
\newpage
\appendix
\begin{center}
\Large\bf Supplementary Materials
\end{center}

\section{Proofs for the Main Results}\label{appendix:proof}
We first sketch the proof in section \ref{app:proofsketch} and give some preliminary lemmas \ref{app:pre}. The following sections \ref{section decomposition}, \ref{section bias} and \ref{section variance} are devoted to the proof of Theorem \ref{main thm}. The proof of Theorem \ref{thm:compat_main} is given in \ref{app:compat}.
\input{text/proofSketch}

\subsection{Preliminaries}\label{app:pre}

The following result comes from \citet{bartlett2020benign}, which bounds the eigenvalues of $\mb X\mb X^\top$.
\begin{lemma}(Lemma 10 in \citet{bartlett2020benign})\label{eigenvalue upperbound}
For any $\sigma_x$, there exists a constant c, such that for any $0\le k<n$, with probability at least $1-e^{-\frac{n}{c}}$,
\begin{equation}\label{eq:eigenvalue upperbound}
    \mu_{k+1}\le c\left(\sum_{i>k}\lambda_i+\lambda_{k+1}n\right).
\end{equation}
\end{lemma}

This implies that as long as the step size $\lambda$ is small than a threshold independent of sample size $n$, gradient descent is stable.
\begin{corollary}\label{lambda bound}
There exists a constant $c$, such that with probability at least $1-e^{-\frac{n}{c}}$, for any $0\le \lambda\le \frac{1}{c\sum_{i>0}\lambda_i}$ we have 
\begin{equation}
\mb O \preceq \mb I-\frac{\lambda}{n}\mb X^\top \mb X \preceq \mb I.    
\end{equation}

\end{corollary}

\begin{proof}
The right hand side of the inequality is obvious since $\lambda>0$. For the left hand side, we have to show that the eigenvalues of $\mb I-\frac{\lambda}{n}\mb X^\top \mb X$ is non-negative. since $\mb X^\top \mb X $ and $\mb X\mb X^\top$ have the same non-zero eigenvalues, we know that with probability at least $1-e^{-\frac{n}{c}}$, the smallest eigenvalue of $\mb I-\frac{\lambda}{n}\mb X^\top \mb X$ can be lower bounded by

\begin{align}
\begin{aligned}
    1-\frac{\lambda}{n}\mu_1\ge 1-c\lambda\left(\frac{\sum_{i>0}\lambda_i}{n}+\lambda_{k+1}\right)\ge 1-2c\lambda \sum_{i>0}\lambda_i\ge 0.
\end{aligned}
\end{align}

where the second inequality uses lemma \ref{eigenvalue upperbound}, and the last inequality holds if $\lambda\le \frac{1}{2c\sum_{i>0}\lambda_i}$.
\end{proof}

\subsection{Proof for the Bias-Variance Decomposition}\label{section decomposition}
Let $\mb X^\dagger$ denote the Moore–Penrose pseudoinverse of matrix $\mb X$.
The following lemma gives a closed form expression for $\mb \theta_t$.
\begin{lemma}\label{lem: theta_t formula}
The dynamics of $\{\mb \theta_t\}_{t\ge 0}$ satisfies 
\begin{equation}
    \mb \theta_t=\left(\mb I-\frac{\lambda}{n}\mb X^{\top}\mb X\right)^t(\mb \theta_0-\mb X^{\dagger}\mb Y)+\mb X^{\dagger}\mb Y.
\end{equation}
\end{lemma}
\begin{proof}
We prove the lemma using induction. The equality holds at $t=0$ as both sides are $\mb \theta_0$. 
Recall that $\mb \theta_t$ is updated as  
\begin{equation}
\mb \theta_{t+1}=\mb \theta_t+\frac{\lambda}{n}\mb X^\top(\mb Y-\mb X\mb \theta_t).    
\end{equation}

Suppose that the dynamic holds up to the $t$-th step. Plug the expression for $\mb \theta_t$ into the above recursion and note that $\mb X^\top \mb X\mb X^\dagger=\mb X^\top$, we get 
\begin{align}
\begin{aligned}
    \mb \theta_{t+1}&=\left(\mb I-\frac{\lambda}{n}\mb X^\top \mb X\right)\mb \theta_t+\frac{\lambda}{n}\mb X^\top \mb Y\\&=\left(\mb I-\frac{\lambda}{n}\mb X^\top \mb X\right)^{t+1}(\mb \theta_0-\mb X^{\dagger}\mb Y)+\left(\mb I-\frac{\lambda}{n}\mb X^\top \mb X\right)\mb X^\dagger \mb Y+\frac{\lambda}{n}\mb X^\top \mb Y\\&
    =\left(\mb I-\frac{\lambda}{n}\mb X^{\top}\mb X\right)^{t+1}(\mb \theta_0-\mb X^{\dagger}\mb Y)+\mb X^{\dagger}\mb Y.
\end{aligned}    
\end{align}
which finishes the proof.
\end{proof}
Next we prove two identities which will be used in further proof.
\begin{lemma}\label{pseudoinverse identity}
The following two identities hold for any matrix $X$ and non-negative integer $t$:
\begin{equation}
\mb I-\mb X^\dagger \mb X+\left(\mb I-\frac{\lambda}{n}\mb X^\top \mb X\right)^t\mb X^\dagger \mb X=\left(\mb I-\frac{\lambda}{n}\mb X^\top \mb X\right)^t,    
\end{equation}

\begin{equation}
\left[\mb I-\left(\mb I-\frac{\lambda}{n}\mb X^\top \mb X\right)^t\right]\mb X^\dagger \mb X\mb X^\top
=\mb X^\top \left[\mb I-\left(\mb I-\frac{\lambda}{n}\mb X \mb X^\top \right)^t\right].    
\end{equation}

\end{lemma}
\begin{proof}
Note that $\mb X^\top \mb X \mb X^\dagger=\mb X^\top$, we can expand the left hand side of the first identity above using binomial theorem and eliminate the pseudo-inverse $\mb X^\dagger$:
\begin{align}
\begin{aligned}
    &\quad \mb I-\mb X^\dagger \mb X+\left(\mb I-\frac{\lambda}{n}\mb X^\top \mb X\right)^t\mb X^\dagger \mb X\\
    &=\mb I-\mb X^\dagger \mb X +\sum_{k=0}^t\binom{t}{k}\left(-\frac{\lambda}{n}\mb X^\top \mb X\right)^k \mb X^\dagger \mb X\\
    &=\mb I-\mb X^\dagger \mb X+ \mb X^\dagger \mb X+\sum_{k=1}^t\binom{t}{k}\left(-\frac{\lambda}{n}\right)^k (\mb X^\top \mb X)^{k-1} \mb X^\top \mb X\mb X^\dagger \mb X\\
    &=\mb I+\sum_{k=1}^t\binom{t}{k}\left(-\frac{\lambda}{n}\right)^k (\mb X^\top \mb X)^{k} \\
    &=\left(\mb I-\frac{\lambda}{n}\mb X^\top \mb X\right)^t.
\end{aligned}    
\end{align}

The second identity can be proved in a similar way:
\begin{align}
\begin{aligned}
    &\quad \left[\mb I-\left(\mb I-\frac{\lambda}{n}\mb X^\top \mb X\right)^t\right]\mb X^\dagger \mb X\mb X^\top\\
    &=-\sum_{k=1}^{t}\binom{t}{k}\left(-\frac{\lambda}{n}\mb X^\top \mb X\right)^k \mb X^\dagger \mb X\mb X^\top\\
    &=-\sum_{k=1}^{t}\binom{t}{k}\left(-\frac{\lambda}{n}\right)^{k}(\mb X^\top \mb X)^{k-1}\mb X^\top \mb X \mb X^\dagger \mb X\mb X^\top\\
    &=-\sum_{k=1}^{t}\binom{t}{k}\left(-\frac{\lambda}{n}\right)^{k}(\mb X^\top \mb X)^{k-1}\mb X^\top \mb X\mb X^\top\\
    &=-\sum_{k=1}^{t}\binom{t}{k}\left(-\frac{\lambda}{n}\right)^{k}\mb X^\top(\mb X\mb X^\top)^{k}\\
    &=\mb X^\top \left[\mb I-\left(\mb I-\frac{\lambda}{n}\mb X \mb X^\top \right)^t\right].
\end{aligned}    
\end{align}

\end{proof}

We are now ready to prove the main result of this section.
\begin{lemma}\label{lem:decomposition}
The excess risk at the $t$-th epoch can be upper bounded as
\begin{equation}
    R(\mb \theta_t)\le \mb \theta^{*\top}\mb B\mb \theta^{*}+\mb \varepsilon^{\top}\mb C\mb \varepsilon,
\end{equation}
where 
\begin{equation}
\mb B=\left(\mb I-\frac{\lambda}{n}\mb X^\top \mb X\right)^t\mb \Sigma \left(\mb I-\frac{\lambda}{n}\mb X^\top \mb X\right)^t,
\end{equation}
\begin{equation}
\mb C=\left(\mb X\mb X^\top\right)^{-1}\left[\mb I-\left(\mb I-\frac{\lambda}{n}\mb X\mb X^\top\right)^t\right]
\mb X\mb \Sigma \mb X^\top\left[\mb I-\left(\mb I-\frac{ \lambda}{n}\mb X\mb X^\top\right)^t\right]\left(\mb X\mb X^\top\right)^{-1},
\end{equation}
which characterizes bias term and variance term in the excess risk.
Furthermore, there exists constant $c$ such that with probability at least $1-\delta$ over the randomness of $\mb \varepsilon$, we have 
\begin{equation}
     \mb \varepsilon^{\top}\mb C\mb \varepsilon\le c\sigma_y^2\log\frac{1}{\delta}\Tr[\mb C].
\end{equation}
\end{lemma}

\begin{proof}
First note that $\mb X \mb X^\top$ is invertible by Assumption \ref{assump:fullRank}. Express the excess risk as follows
\begin{align}
\begin{aligned}
    R(\mb\theta_t)&=\frac{1}{2}\ee[(y-\mb x^\top \mb\theta_t)^2-(y-\mb x^\top \mb\theta^*)^2]\\
    &=\frac{1}{2}\ee[(y-\mb x^\top \mb\theta^* +\mb x^\top \mb\theta^* -\mb x^\top \mb\theta_t)^2-(y-\mb x^\top \mb\theta^*)^2]\\
    &=\frac{1}{2}\ee[(\mb x^\top(\mb\theta_t-\mb\theta^*))^2+2(y-\mb x^\top \mb\theta^*)(\mb x^\top \mb\theta^* -\mb x^\top \mb\theta_t)]\\
    &=\frac{1}{2}\ee[\mb x^\top(\mb\theta_t-\mb\theta^*)]^2.
\end{aligned}    
\end{align}

Recall that $\mb\theta_0=0$ and $\mb Y=\mb X\mb\theta^*+\mb \varepsilon$ and we can further simplify the formula for $\mb \theta_t$ in lemma \ref{lem: theta_t formula}:
\begin{align}
\begin{aligned}
    \mb\theta_t
    &=\left(\mb I-\frac{\lambda}{n}\mb X^{\top}\mb X\right)^t(\mb\theta_0-\mb X^{\dagger}\mb Y)+\mb X^{\dagger}\mb Y\\
    &=\left[\mb I-\left(\mb I-\frac{\lambda}{n}\mb X^{\top}\mb X\right)^t\right]\mb X^\dagger(\mb X\mb\theta^*+\mb \varepsilon).    
\end{aligned}
\end{align}

Plug it into the above expression for $R(\mb \theta_t)$, we have
\begin{align}
\begin{aligned}
    R(\mb\theta_t)&=\frac{1}{2}\ee\left[\mb x^\top \left[\mb I-\left(\mb I-\frac{\lambda}{n}\mb X^{\top}\mb X\right)^t\right]\mb X^\dagger(\mb X\mb\theta^*+\mb\varepsilon)-\mb x^\top \mb\theta^*\right]^2\\
    &=\frac{1}{2}\ee\left[\mb x^\top\left(\mb X^\dagger \mb X-\left(\mb I-\frac{\lambda}{n}\mb X^{\top}\mb X\right)^t\mb X^\dagger \mb X-\mb I\right)\mb\theta^*
    \right. \\& \left.\quad+\mb x^\top
    \left[\mb I-\left(\mb I-\frac{\lambda}{n}\mb X^{\top}\mb X\right)^t\right]\mb X^\dagger\mb \varepsilon\right]^2 .
\end{aligned}    
\end{align}

Applying lemma \ref{pseudoinverse identity}, we obtain
\begin{align}
\begin{aligned}
    R(\mb \theta_t)&=\frac{1}{2}\ee\left[-\mb x^\top\left(\mb I-\frac{\lambda}{n}\mb X^\top \mb X\right)^t\mb\theta^*
    +\mb x^\top \mb X^\top \left[\mb I-\left(\mb I-\frac{\lambda}{n}\mb X \mb X^\top \right)^t\right](\mb X\mb X^\top )^{-1}\mb\varepsilon\right]^2\\
    &\le \ee\left[\mb x^\top\left(\mb I-\frac{\lambda}{n}\mb X^\top \mb X\right)^t\mb\theta^*
    \right]^2+ \ee\left[\mb x^\top \mb X^\top \left[\mb I-\left(\mb I-\frac{\lambda}{n}\mb X \mb X^\top \right)^t\right](\mb X\mb X^\top )^{-1}\mb\varepsilon\right]^2
    \\&:=\mb\theta^{*\top}\mb B\mb\theta^{*}+\mb\varepsilon^{\top}\mb C\mb\varepsilon.
\end{aligned}    
\end{align}
which proves the first claim in the lemma. The second part of the theorem directly follows from lemma 18 in \citet{bartlett2020benign}.
\end{proof}

\subsection{Proof for the Bias Upper Bound}\label{section bias}
The next lemma guarantees that the sample covariace matrix $\frac{1}{n}\mb X^\top \mb X$ concentrates well around $\mb \Sigma$.
\begin{lemma}\label{covariance}
(Lemma 35 in \citet{bartlett2020benign}) There exists constant $c$ such that  for any $0<\delta<1$ with probability as least $1-\delta$,
\begin{equation}
\left\|\mb \Sigma-\frac{1}{n}\mb X^\top \mb X\right\|\le 
c\|\mb \Sigma\|\max
\left\{\sqrt{\frac{r(\mb \Sigma)}{n}},\frac{r(\mb \Sigma)}{n},\sqrt{\frac{\log(\frac{1}{\delta})}{n}},\frac{\log(\frac{1}{\delta})}{n}\right\}.
\end{equation}
\end{lemma}

The following inequality will be useful in our proof to characterize the decaying rate of the bias term with $t$.
\begin{lemma}\label{product of P}
For any positive semidefinite matrix $\mb P$ which satisfies $\|\mb P\|\le 1$, we have 
\begin{equation}
\|\mb P(1-\mb P)^t\|\le \frac{1}{t}.
\end{equation}
\end{lemma}
\begin{proof}
Assume without loss of generality that $\mb P$ is diagonal. Then it suffices to consider seperately each eigenvalue $\sigma$ of $\mb P$, and show that $\sigma(1-\sigma)^t\le\frac{1}{t}$.

In fact, by AM-GM inequality we have
\begin{equation}
\sigma(1-\sigma)^t\le\frac{1}{t}\left[\frac{t\sigma+(1-\sigma)t}{t+1}\right]^{t+1}\le\frac{1}{t},
\end{equation}
which completes the proof.
\end{proof}

Next we prove the main result of this section.
\begin{lemma}\label{lem:bias_main}
There exists constant $c$ such that if $0\le \lambda\le \frac{1}{c\sum_{i>0}\lambda_i}$, then for any $0<\delta<1$, with probability at least $1-\delta$ the following bound on the bias term holds for any $t$
\begin{equation}
\mb\theta^{*\top} \mb B\mb\theta^*\le c\left\|\mb\theta^*\right\|^2\left(\frac{1}{\lambda t}+\|\mb \Sigma\|
    \max\left\{\sqrt{\frac{r(\mb \Sigma)}{n}},\frac{r(\mb \Sigma)}{n},\sqrt{\frac{\log(\frac{1}{\delta})}{n}},\frac{\log(\frac{1}{\delta})}{n}\right\}\right).
\end{equation}
\end{lemma}

\begin{proof}
The bias can be decomposed into the following two terms 
\begin{align}\label{opt1}
\begin{aligned}
    \mb\theta^{*\top} \mb B\mb\theta^*&=\mb\theta^{*\top}\left(\mb I-\frac{\lambda}{n}\mb X^\top \mb X\right)^t\left(\mb \Sigma-\frac{1}{n}\mb X^\top \mb X\right)\left(\mb I-\frac{\lambda}{n}\mb X^\top \mb X\right)^t\mb\theta^*
    \\&\quad+\mb\theta^{*\top}\left(\frac{1}{n}\mb X^\top \mb X\right)\left(\mb I-\frac{\lambda}{n}\mb X^\top \mb X\right)^{2t}\mb\theta^*.
\end{aligned}    
\end{align}

For sufficiently small learning rate $\lambda$ as given by corollary \ref{lambda bound}, we know that with high probability
\begin{equation}
\left\|\mb I-\frac{\lambda}{n}\mb X^\top \mb X\right\|\le 1,    
\end{equation}
which together with lemma \ref{covariance} gives a high probability bound on the first term:
\begin{align}
\begin{aligned}
    &\quad\mb\theta^{*\top}\left(\mb I-\frac{\lambda}{n}\mb X^\top \mb X\right)^t\left(\mb \Sigma-\frac{1}{n}\mb X^\top \mb X\right)\left(\mb I-\frac{\lambda}{n}\mb X^\top \mb X\right)^t\mb\theta^*
    \\&\le c\|\mb \Sigma\|\left\|\mb\theta^*\right\|^2
    \max\left\{\sqrt{\frac{r(\mb \Sigma)}{n}},\frac{r(\mb \Sigma)}{n},\sqrt{\frac{\log(\frac{1}{\delta})}{n}},\frac{\log(\frac{1}{\delta})}{n}\right\}.
\end{aligned}    
\end{align}

For the second term, invoke lemma \ref{product of P} with $\mb P=\frac{\lambda}{n}\mb X^\top \mb X$ and we get
\begin{align}
\begin{aligned}
    \mb\theta^{*\top}\left(\frac{1}{n}\mb X^\top \mb X\right)\left(\mb I-\frac{\lambda}{n}\mb X^\top \mb X\right)^{2t}\mb\theta^*
    &\le \frac{1}{\lambda}\left\|\mb\theta^*\right\|^2\left\|\left(\frac{\lambda}{n}\mb X^\top \mb X\right)\left(\mb I-\frac{\lambda}{n}\mb X^\top \mb X\right)^{2t}\right\|\\
    &\le \frac{1}{2\lambda t}\left\|\mb\theta^*\right\|^2.
\end{aligned}    
\end{align}

Putting these two bounds together gives the proof for the main theorem.

\end{proof}

\subsection{Proof for the Variance Upper Bound}\label{section variance}
Recall that $\mb X=\mb U\Tilde{\mb \Lambda}^\frac{1}{2}\mb W^\top$ is the singular value decomposition of data matrix $\mb X$, where $\mb U=(\mb u_1,\cdots, \mb u_n)$, $\mb W=(\mb w_1,\cdots, \mb w_n)$, $\Tilde{\mb \Lambda}=\text{diag}\{\mu_1,\cdots,\mu_n\}$ with $\mu_1\ge\mu_2\ge\cdots\mu_n$. 

Recall that 
\begin{align}
\begin{aligned}
    &k_0=\min\{l\ge 0:\lambda_{l+1}\le \frac{c_0\sum_{i>l}\lambda_i}{n}\},\\
    &k_1 = \min\{l\ge 0:\lambda_{l+1}\le \frac{c_1\sum_{i>0}\lambda_i}{n}\},\\  
    &k_2=\min\{l\ge 0:\sum_{i>l}\lambda_i+n\lambda_{l+1}\le c_2 c(t,n) \sum_{i>0}\lambda_i\}\},  
\end{aligned}
\end{align}

for some constant $c_0,c_1,c_2$ and function $c(t,n)$.

We further define 
\begin{equation}
    k_3=\min\{l\ge 0:\mu_{l+1}\le c_3 c(t,n) \sum_{i>0}\lambda_i\},
\end{equation}
for some constant $c_3$.

The next lemma shows that we can appropriately choose constants to ensure that $k_3\le k_2$ holds with high probability, and in some specific cases we have $k_2\le k_1$.
\begin{lemma}\label{k order}
For any function $c(t,n)$ and constant $c_2$, there exists constants $c,c_3$, such that $k_3\le k_2$ with probability at least $1-e^{-\frac{n}{c}}$. Furthermore, if $c(t,n)$ is a positive constant function, for any $c_1$, there exists $c_2$ such that $k_2\le k_1$.
\end{lemma}
\begin{proof}
According to lemma \ref{eigenvalue upperbound}, there exists a constant $c$, with probability at least $1-e^{-\frac{n}{c}}$ we have 
\begin{equation}
    \mu_{k_2+1}\le c(\sum_{i>k_2}\lambda_i+n\lambda_{k_2+1})\le c c_2 c(t,n)\sum_{i>0}\lambda_i.
\end{equation}
Therefore, we know that $k_3\le k_2$ for $c_3=c c_2$.

By the definition of $k_1$, we have 
\begin{equation}
\sum_{i>k_1}\lambda_i+n\lambda_{k_1+1}\le (c_1+1) \sum_{i>0}\lambda_i,    
\end{equation}
which implies that $k_2\le k_1$ for $c_2=\frac{c_1+1}{c(t,n)}$, if $c(t,n)$ is a positive constant.
\end{proof}

Next we bound $\Tr[\mb C]$, which implies an upper bound on the variance term.

\begin{theorem}\label{thm:var_main}
There exist constants $c,c_0,c_1,c_2$ such that if $k_0\le \frac{n}{c}$, then with probability at least $1-e^{-\frac{n}{c}}$, the trace of the variance matrix $C$ has the following upper bound for any $t$:
\begin{equation}
\Tr[\mb C]\le c\left(\frac{k_1}{n}+\frac{k_2}{c(t,n)n}+
c(t,n)\left(\frac{\lambda t}{n}\sum_{i>0}\lambda_i\right)^2\right).
\end{equation}

\end{theorem}

\begin{proof}

We divide the eigenvalues of $\mb X\mb X^\top$ into two groups based on whether they are greater than $c_3 c(t,n) \sum_{i>0}\lambda_i$. The first group consists of $\mu_1\cdots \mu_{k_3}$, and the second group consists of $\mu_{k_3+1}\cdots \mu_n$. For $1\le j\le k_3$, we have 
\begin{equation}
1-\left(1-\frac{\lambda}{n}\mu_j\right)^t\le 1.    
\end{equation}

Therefore we have the following upper bound on $\left[\mb I-\left(\mb I-\frac{\lambda}{n}\mb X\mb X^\top\right)^t\right]^2$:

\begin{align}
    \begin{aligned}
    &\quad\left[\mb I-\left(\mb I-\frac{\lambda}{n}\mb X\mb X^\top\right)^t\right]^2
    \\&=\mb U\text{diag}\left\{\left[1-\left(1-\frac{\lambda}{n}\mu_1\right)^t\right]^2\cdots \left[1-\left(1-\frac{\lambda}{n}\mu_n\right)^t\right]^2\right\}\mb U^\top\\
    &\preceq \mb U\text{diag}\left\{\overbrace{1,\cdots 1}^{k_3 \text{ times}},\overbrace{\left[1-\left(1-\frac{\lambda}{n}\mu_{k_3+1}\right)^t\right]^2,\cdots \left[1-\left(1-\frac{\lambda}{n}\mu_n\right)^t\right]^2}^{n-k_3 \text{ times}}\right\}\mb U^\top\\
    &= \mb U\text{diag}\left\{\overbrace{1,\cdots 1}^{k_3 \text{ times}},\overbrace{0,\cdots 0}^{n-k_3 \text{ times}}\right\}\mb U^\top\\
    &+\mb U\text{diag}\left\{\overbrace{0,\cdots 0}^{k_3 \text{ times}},\overbrace{\left[1-\left(1-\frac{\lambda}{n}\mu_{k_3+1}\right)^t\right]^2,\cdots \left[1-\left(1-\frac{\lambda}{n}\mu_n\right)^t\right]^2}^{n-k_3 \text{ times}}\right\}\mb U^\top.
    \end{aligned}
\end{align}

For positive semidefinite matrices $\mb P,\mb Q,\mb R$ which satisfies $\mb Q\preceq \mb R$, it holds that $\Tr[\mb P\mb Q]\le \Tr[\mb P\mb R]$ . It implies the following upperbound of $\Tr[\mb C]$:

\begin{equation}
    \begin{split}\label{cut 1}
        &\quad\Tr[\mb C]\\&=\Tr\left[\left[\mb I-\left(\mb I-\frac{\lambda}{n}\mb X\mb X^\top\right)^t\right]^2\left(\mb X\mb X^\top\right)^{-2}\mb X\mb \Sigma \mb X^\top\right]\\
    &\le \underbrace{\Tr\left[\mb U\text{diag}\left\{\overbrace{1,\cdots 1}^{k_3 \text{ times}},\overbrace{0,\cdots 0}^{n-k_3 \text{ times}}\right\}\mb U^\top\left(\mb X\mb X^\top\right)^{-2}\mb X\mb \Sigma \mb X^\top\right]}_{\text{\textcircled{1}}}\\
    &+\underbrace{\Tr\left[\mb U\text{diag}\left\{\overbrace{0,\cdots 0}^{k_3 \text{ times}},\overbrace{\left[1-\left(1-\frac{\lambda}{n}\mu_{k_3+1}\right)^t\right]^2,\cdots \left[1-\left(1-\frac{\lambda}{n}\mu_n\right)^t\right]^2}^{n-k_3 \text{ times}}\right\}
    \mb U^\top\left(\mb X\mb X^\top\right)^{-2}\mb X\mb \Sigma \mb X^\top\right].}_{\text{\textcircled{2}}}
    \end{split}
\end{equation}

\textbf{Bounding \textcircled{1}}

Noticing $\mb X=\mb U\Tilde{\mb \Lambda}^\frac{1}{2}\mb W^\top$ and $\mb \Sigma=\sum_{i\ge 1}\lambda_i \mb v_i\mb v_i^\top$, we express the first term as sums of eigenvector products,
\begin{align}
    \begin{aligned}
    \text{\textcircled{1}}&=\Tr\left[\mb U\text{diag}\left\{\overbrace{1,\cdots 1}^{k_3 \text{ times}},\overbrace{0,\cdots 0}^{n-k_3 \text{ times}}\right\}\mb U^\top\left(\mb X\mb X^\top\right)^{-2}\mb X\mb \Sigma \mb X^\top\right]\\
    &=\Tr\left[\mb U\text{diag}\left\{\overbrace{1,\cdots 1}^{k_3 \text{ times}},\overbrace{0,\cdots 0}^{n-k_3 \text{ times}}\right\}\mb U^\top\mb U\Tilde{\mb \Lambda}^{-2}\mb U^\top \mb U\Tilde{\mb \Lambda}^{\frac{1}{2}}\mb W^\top \mb \Sigma \mb W\Tilde{\mb \Lambda}^{\frac{1}{2}}\mb U^\top\right]\\
    &=\Tr\left[\text{diag}\left\{\overbrace{1,\cdots 1}^{k_3 \text{ times}},\overbrace{0,\cdots 0}^{n-k_3 \text{ times}}\right\}\Tilde{\mb \Lambda}^{-1}\mb W^\top \mb \Sigma \mb W\right]\\
    &=\sum_{i\ge 1} \lambda_i \Tr\left[\text{diag}\left\{\overbrace{1,\cdots 1}^{k_3 \text{ times}},\overbrace{0,\cdots 0}^{n-k_3 \text{ times}}\right\}\Tilde{\mb \Lambda}^{-1}\mb W^\top \mb v_i\mb v_i^\top \mb W\right]\\
    &=\sum_{i\ge 1}\sum_{1\le j\le k_3}\frac{\lambda_i}{\mu_j}\left(\mb v_i^\top \mb w_j\right)^2.
    \end{aligned}
\end{align}

Next we divide the above summation into $1\le i\le k_1$ and $i>k_1$. 
For the first part, notice that 
\begin{align}
    \begin{aligned}
    \sum_{1\le j\le k_3}\frac{\lambda_i}{\mu_j}\left(\mb v_i^\top \mb w_j\right)^2
    &\le \sum_{1\le j\le n}\frac{\lambda_i}{\mu_j}\left(\mb v_i^\top \mb w_j\right)^2\\
    &=\lambda_i \mb v_i^\top \left(\sum_{1\le j\le n}\frac{1}{\mu_j}\mb w_j\mb w_j^\top\right) \mb v_i\\
    &=\lambda_i \mb v_i^\top \mb W\Tilde{\mb \Lambda}^{-1}\mb W^\top \mb v_i\\
    &=\lambda_i \mb v_i^\top \mb W\Tilde{\mb \Lambda}^{\frac{1}{2}}\mb U^\top \mb U \Tilde{\mb \Lambda}^{-2}\mb U^\top \mb U
    \Tilde{\mb \Lambda}^{\frac{1}{2}}\mb W^\top \mb v_i\\
    &=\lambda_i^2 \Tilde{\mb x}_i^\top(\mb X\mb X^\top)^{-2}\Tilde{\mb x}_i,
    \end{aligned}
\end{align}
where $\Tilde{\mb x}_i$ is defined as $\Tilde{\mb x}_i=\frac{\mb X\mb v_i}{\sqrt{\lambda_i}}=\frac{\mb U
    \Tilde{\mb \Lambda}^{\frac{1}{2}}\mb W^\top \mb v_i}{\sqrt{\lambda_i}}$.

From the proof of lemma 11 in \citet{bartlett2020benign}, we know that for any $\sigma_x$, there exists a constant $c_0$ and $c$ such that if $k_0\le \frac{n}{c}$, with probability at least $1-e^{-\frac{n}{c}}$ the first part can be bounded as 

\begin{align}\begin{aligned}\label{var first part }
    \sum_{1\le i\le k_1}\sum_{1\le j\le k_3}\frac{\lambda_i}{\mu_j}\left(\mb v_i^\top \mb w_j\right)^2 \le \sum_{1\le i\le k_1}\lambda_i^2 \mb \Tilde{\mb x}_i(\mb X\mb X^\top)^{-2}\Tilde{\mb x}_i\le c\frac{k_1}{n},
\end{aligned}\end{align}
which gives a bound for the first part.

For the second part we interchange the order of summation and have 
\begin{align}\begin{aligned}\label{var second part}
    \sum_{i\ge k_1}\sum_{1\le j\le k_3}\frac{\lambda_i}{\mu_j}\left(\mb v_i^\top \mb w_j\right)^2
    &=\sum_{1\le j\le k_3}\sum_{i\ge k_1}\frac{\lambda_i}{\mu_j}\left(\mb v_i^\top \mb w_j\right)^2\\
    &\le \frac{1}{c_3 c(t,n)\sum_{i>0}\lambda_i}\sum_{1\le j\le k_3}\sum_{i\ge k_1}\lambda_i\left(\mb v_i^\top \mb w_j\right)^2\\
    &=\frac{\lambda_{k_1+1}}{c_3 c(t,n)\sum_{i>0}\lambda_i}\sum_{1\le j\le k_3}\sum_{i\ge k_1}\left(\mb v_i^\top \mb w_j\right)^2\\
    &\le \frac{\lambda_{k_1+1}}{c_3 c(t,n)\sum_{i>0}\lambda_i}\sum_{1\le j\le k_3}1\\
    &=\frac{\lambda_{k_1+1}k_3}{c_3 c(t,n)\sum_{i>0}\lambda_i}\\
    &\le c \frac{k_3}{c(t,n)n}.
\end{aligned}\end{align}

for $c$ large enough.

Putting \ref{var first part } and \ref{var second part} together, and noting that $k_3\le k_2$ with high probability as given in lemma \ref{k order}, we know there exists a constant $c$ that with probability at least $1-e^{-\frac{n}{c}}$,
\begin{equation}
    \text{\textcircled{1}}\le c\frac{k_1}{n}+c \frac{k_2}{c(t,n)n}.
\end{equation}

\textbf{Bounding \textcircled{2}}

Similar to the first step in bounding \textcircled{1}, we note that 
\begin{align}
    \begin{aligned}
    \text{\textcircled{2}}&=
    \Tr\left[\mb U\text{diag}\left\{\overbrace{0,\cdots 0}^{k_3 \text{ times}},\overbrace{\left[1-\left(1-\frac{\lambda}{n}\mu_{k_3+1}\right)^t\right]^2,\cdots, \left[1-\left(1-\frac{\lambda}{n}\mu_n\right)^t\right]^2}^{n-k_3 \text{ times}}\right\}\right. \\&\left.\quad\quad \mb U\Tilde{\mb \Lambda}^{-2}\mb U^\top \mb U\Tilde{\mb \Lambda}^{\frac{1}{2}}\mb W^\top \mb \Sigma \mb W\Tilde{\mb \Lambda}^{\frac{1}{2}}\mb U^\top\right]\\
    &=
    \Tr\left[\text{diag}\left\{\overbrace{0,\cdots 0}^{k_3 \text{ times}},\overbrace{\frac{1}{\mu_{k_3+1}}\left[1-\left(1-\frac{\lambda}{n}\mu_{k_3+1}\right)^t\right]^2,\cdots, \frac{1}{\mu_n}\left[1-\left(1-\frac{\lambda}{n}\mu_n\right)^t\right]^2}^{n-k_3 \text{ times}}\right\}\right. \\&\left.\quad\quad\mb W^\top \mb \Sigma \mb W\right].
    \end{aligned}
\end{align}

From Bernoulli's inequality and the definition of $k_3$, for any $k_3+1\le j\le n$, we have
\begin{align}\label{bernoulli}
    \frac{1}{\mu_k}\left[1-\left(1-\frac{\lambda}{n}\mu_k\right)^t\right]^2\le
    \frac{1}{\mu_k}\left(\frac{\lambda}{n}\mu_kt\right)^2
    =\left(\frac{\lambda t}{n}\right)^2\mu_k\le c_3 \left(\frac{\lambda t}{n}\right)^2 c(t,n)\sum_{i>0}\lambda_i,
\end{align}
Hence, 

\begin{align}
    \begin{aligned}
        \text{\textcircled{2}}&\le c_3 c(t,n)\left(\frac{\lambda t}{n}\right)^2 \sum_{i>0}\lambda_i \Tr[\mb W^\top \mb \Sigma \mb W]
    \\&=c_3 c(t,n)\left(\frac{\lambda t}{n}\sum_{i>0}\lambda_i\right)^2.
    \end{aligned}
\end{align}

\textbf{Putting things together}

From the bounds for \textcircled{1} and \textcircled{2} given above, we know that there exists a constant $c$ such that with probability at least $1-e^{-\frac{n}{c}}$, the trace of the variance matrix $C$ has the following upper bound
\begin{align}
    \begin{aligned}
            \Tr[C]\le c\left(\frac{k_1}{n}+\frac{k_2}{c(t,n)n}+
    c(t,n)\left(\frac{\lambda t}{n}\sum_{i>0}\lambda_i\right)^2\right).
    \end{aligned}
\end{align}

\end{proof}

\begin{proof}[Proof of theorem \ref{main thm}]
Lemma \ref{lem:decomposition}, \ref{lem:bias_main} and Theorem \ref{thm:var_main} gives the complete proof. Note that the high probability events in the proof are independent of the epoch number $t$, and this implies that the theorem holds uniformly for all $t\in\Nbb$.
\end{proof}

\subsection{Proof of Compatibility Results}\label{app:compat}

\begin{corollary}[Corollary \ref{sqrt n} restated]
Let Assumption \ref{assumption 1}, \ref{assumption 2} and \ref{assump:fullRank} hold. Fix a constant $c(t, n)$. Suppose $k_0= O(n)$, $k_1=o(n)$, $r(\mb\Sigma)=o(n)$, $\lambda=O\left(\frac{1}{\sum_{i>0}\lambda_i}\right)$.
Then there exists a sequence of positive constants $\{\delta_n\}_{n\ge 0}$ which converge to 0, such that with probability at least $1-\delta_n$, the excess risk is consistent for $t\in\left(\omega\left(\frac{1}{\lambda}\right),o\left(\frac{n}{\lambda}\right)\right)$, i.e. 
\begin{equation*}
    R(\mb\theta_t)=o(1).
\end{equation*}
Furthermore, for any positive constant $\delta$, with probability at least $1-\delta$, the minimal excess risk on the training trajectory can be bounded as
\begin{equation*}
    \min_t R(\mb\theta_t)\lesssim \frac{\max\{\sqrt{r(\mb\Sigma)},1\}}{\sqrt{n}}
+\frac{\max\{k_1,1\}}{n}.
\end{equation*}
\end{corollary}

\begin{proof}
According to Lemma \ref{lem:bias_main}, with probability at least $1-\frac{\delta_n}{2}$, the following inequality holds for all $t$:
\begin{equation}\label{eq: bias bound}
B(\mb\theta_t)\lesssim \left(\frac{1}{\lambda t}+
\max\left\{\sqrt{\frac{r(\mb\Sigma)}{n}},
\frac{r(\mb\Sigma)}{n},\sqrt{\frac{\log(\frac{1}{\delta_n})}{n}},\frac{\log(\frac{1}{\delta_n})}{n}\right\}\right).    
\end{equation}
If $\delta_n$ is chosen such that $\log \frac{1}{\delta_n}=o(n)$, we have that with probability at least $1-\frac{\delta_n}{2}$, we have for all $t=\omega\left(\frac{1}{\lambda}\right)$:
\begin{equation}
    B(\mb\theta_t)=o(1),
\end{equation}
in the sample size $n$.

When $c(t,n)$ is a constant, we have $k_2\le k_1$ with high probability as given in lemma \ref{k order}.
Therefore, according to Lemma~\ref{lem:decomposition} and Theorem~\ref{thm:var_main}, we know that if $\log \frac{1}{\delta_n}=O(n)$, with probability at least $1-\frac{\delta_n}{2}$, the following bound holds for all $t$:
\begin{equation}\label{eq: var bound}
    V(\mb\theta_t)\lesssim \log\left(\frac{1}{\delta_n}\right)\left(\frac{k_1}{n}+\frac{\lambda^2 t^2}{n^2}\right).
\end{equation}

Since $k_1=o(n)$, $t=o\left(\frac{n}{\lambda}\right)$, we have
$
    \frac{k_1}{n}+\frac{\lambda^2 t^2}{n^2}=o(1).
$
Therefore, there exists a mildly decaying sequence of $\delta_n$ with $\log\left(\frac{1}{\delta_n}\right)\left(\frac{k_1}{n}+\frac{\lambda^2 t^2}{n^2}\right)=o(1)$, i.e.,
\begin{equation}
    V(\mb\theta_t)=o(1).
\end{equation}

To conclude, $\delta_n$ can be chosen such that 
\begin{equation}
    \log\left(\frac{1}{\delta_n}\right)=\omega(1), \log\left(\frac{1}{\delta_n}\right)=O(n), \log\left(\frac{1}{\delta_n}\right)=O\left(\frac{1}{\frac{k_1}{n}+\frac{\lambda^2 t^2}{n^2}}\right),
\end{equation}
and then with probability at least $1-\delta_n$, the excess risk is consistent for all 
$t\in\left(\omega\left(\frac{1}{\lambda}\right),o\left(\frac{n}{\lambda}\right)\right)$:
\begin{equation}
    R(\mb\theta_t)=B(\mb\theta_t)+V(\mb\theta_t)=o(1).
\end{equation}

This completes the proof for the first claim. The second claim follows from 
Equation~\ref{eq: bias bound} and~\ref{eq: var bound} by setting $t=\Theta\left(\frac{\sqrt{n}}{\lambda}\right)$.
\end{proof}

\begin{lemma}[Lemma \ref{lem:k1} restated]
For any fixed (i.e. independent of sample size $n$) feature covariance $\mb \Sigma$ satisfying assumption \ref{assumption 1}, we have $k_1(n)=o(n)$.
\end{lemma}
\begin{proof}
Suppose there exists constant $c$, such that $k_1(n)\ge c n$. By definition of $k_1$, we know that $\lambda_l\ge \frac{c_1 \sum_{i>0}\lambda_i}{n}$ holds for $1\le l\le k_1(n)$. Hence we have
\begin{equation}
\sum_{l=\lfloor cn2^i\rfloor+1}^{\lfloor cn 2^{i+1}\rfloor}\lambda_l \gtrsim \frac{c_1\sum_{i>0}\lambda_i}{n 2^{i+1}}cn 2^i\gtrsim \sum_{i>0}\lambda_i.  
\end{equation}
summing up all $l$ leads to a contradiction since $\sum_{i>0}\lambda_i<\infty$, which finishes the proof.
\end{proof}

\begin{theorem}[Theorem \ref{thm:compat_main} restated]
Consider the overparameterized linear regression setting defined in section \ref{sec: linreg_def}. 
Let Assumption \ref{assumption 1},\ref{assumption 2} and \ref{assump:fullRank} hold. 
Assume the learning rate satisfies $\lambda=O\left(\frac{1}{\Tr(\mb \Sigma)}\right)$.
\begin{itemize}
    \item If the covariance satisfies $k_0=o(n),R_{k_0}(\mb \Sigma)=\omega(n), \ r(\mb \Sigma)=o(n)$, it is compatible with the region $T_n=\left(\omega\left(\frac{1}{\lambda}\right), \infty\right)$.
    \item If the covariance satisfies $k_0=O(n),k_1=o(n),r(\mb \Sigma)=o(n)$, it is compatible with the region $T_n=\left(\omega\left(\frac{1}{\lambda}\right),o\left(\frac{n}{\lambda}\right)\right)$.
    \item If the covariance does not change with $n$, and satisfies $k_0=O(n)$ and $p=\infty$, it is compatible with the region $T_n=\left(\omega\left(\frac{1}{\lambda}\right),o\left(\frac{n}{\lambda}\right)\right)$.
\end{itemize}
\end{theorem}

\begin{proof}
For the first argument, notice that (a) the bias term can still be bounded when $t = \omega(1)$ and $r(\Sigma) = o(n)$, according to Lemma~\ref{section bias}; (b) the variance term can be bounded with $t \to \infty$ (that is, the variance loss would increase with time $t$).
Therefore, the first argument directly follows Theorem 4 in~\citet{bartlett2020benign}.

The second argument follows Corollary~\ref{sqrt n}, 
and the third argument follows Corollary~\ref{sqrt n} and Lemma~\ref{lem:k1}.
Specifically, for any $\varepsilon>0$, there exists $\{\delta_n\}_{n>0}$ and $N$ such that for any sample size $n>N$, we have
\begin{equation}
    \Pr\left[\left|\sup_{t\in T_n} R(\mb \theta_t)\right|>\varepsilon\right]\le \delta_n.
\end{equation}
Let $n\to \infty$ shows that $\sup_{t\in T_n} R(\mb \theta_t)$ converges to $0$ in probability, which completes the proof for the second and the third claim. 
\end{proof}

%% file: text/proofSketch.tex
\subsection{Proof Sketch}\label{app:proofsketch}

We start with a standard bias-variance decomposition following~\citet{bartlett2020benign}, which derives that the time-variant excess risk $R(\mb \theta_t)$ can be bounded by a bias term and a variance term.
We refer to Appendix~\ref{section decomposition} for more details.

For the bias part, we first decompose it into an optimization error and an approximation error.
For the optimization error, we use the spectrum analysis to bound it with $O\left(1/t\right)$ where $t$ denotes the time.
For the approximation error, we bound it with  $O\left(1/\sqrt{n}\right))$ where $n$ denotes the sample size, inspired by \citet{bartlett2020benign}.
We refer to Appendix~\ref{section bias} for more details.

For the variance part, a key step is to bound the term $(\mb I-\frac{\lambda}{n}\mb X\mb X^\top)^t$, where $\mb X$ is the feature matrix.
The difficulty arises from the different scales of the eigenvalues of $\mb X\mb X^\top$, where the largest eigenvalue has order $\Theta(n)$ while the smallest eigenvalue has order $O(1)$, according to Lemma~10 in \citet{bartlett2020benign}.
To overcome this issue, we divide the matrix  $\mb X\mb X^\top$ based on whether its eigenvalues is larger than $c(t,n)$, which is a flexible term dependent on time $t$ and sample size $n$.
Therefore, we split the variance term based on eigenvalues of covariance matrix $\mb \Sigma$ (leading to the $k_1$-related term) and based on the eigenvalues of $\mb X\mb X^\top$  (leading to the $k_2$-related term).
We refer to Appendix~\ref{section variance} for more details.

%% file: text/appendix_example.tex
\section{Comparisons and Discussions}\label{appen:exa}
In this section, we provide additional discussions and calculations for the main results, and compare it with 
previous works, including benign overfitting~(Section~\ref{appendix: compare benign overfitting}), stability-based bounds~(Section~\ref{app: stability}), uniform convergence~(Section~\ref{app: uniform}), and early-stopping bounds~(Section~\ref{app: early stopping}).

\subsection{Comparisons with Benign Overfitting}
\label{appendix: compare benign overfitting}

We summarize the results in \citet{bartlett2020benign,zou2021benign} and our results in Table~\ref{tab:comparison}, and provide a detailed comparison with them below.

\textbf{Comparison to \citet{bartlett2020benign}.} In this seminal work, the authors study the excess risk of the min-norm interpolator.
As discussed before, gradient descent converges to the min-norm interpolator in the  overparameterized linear regression setting.
One of the main results in \citet{bartlett2020benign} is to provide a tight bound for the variance part
in excess risk as 
\begin{equation}
\label{eqn: bartlett et al}
V(\boldsymbol{\hat{\theta}}) = O\left(\frac{k_0}{n}+\frac{n}{R_{k_0}(\mb\Sigma)}\right), 
\end{equation}
where $\boldsymbol{\hat{\theta}} = \mb X^\top(\mb X\mb X^\top)^{-1} \mb Y$ denotes the min-norm interpolator, and $R_{k}(\mb\Sigma)={(\sum_{i>k}\lambda_i)^2}/\\({\sum_{i>k}\lambda_i^2})$ denote another type of effective rank.  

By introducing the time factor, Theorem~\ref{main thm} improves over Equation~\eqref{eqn: bartlett et al} in at least two aspects.
Firstly, Theorem~\ref{main thm} guarantees the consistency of the gradient descent dynamics for a broad range of step number $t$, while \citet{bartlett2020benign} study the limiting behavior of the dynamics of $t\to \infty$.
Secondly, Theorem~\ref{main thm} implies that the excess risk of early stopping gradient descent solution can be much better than the min-norm interpolator. 
Compared to the bound in Equation~\eqref{eqn: bartlett et al}, the bound in Corollary~\ref{sqrt n} (a.) replaces $k_0$ with a much smaller quantity $k_1$; and (b.) drops the second term involving $R_{k_0}(\mb\Sigma)$.
Therefore, we can derive a consistent bound for an early stopping solution, even though the excess risk of limiting point (min-norm interpolator) can be $\Omega(1)$.

\textbf{Comparison to \citet{zou2021benign}.}
\citet{zou2021benign} study a different setting, which focuses on the one-pass stochastic gradient descent solution of linear regression.
The authors prove a bound for the excess risk as
\begin{equation}
    \label{eqn: zou et al}
R(\Tilde{\boldsymbol{\theta}}_t) = O\left(\frac{k_1}{n}+\frac{n\sum_{i>k_1}\lambda_i^2}{(\sum_{i>0}\lambda_i)^2}\right),
\end{equation}
where $\Tilde{\boldsymbol{\theta}}_t$ denotes the parameter obtained using stochastic gradient descent (SGD) with constant step size at epoch $t$.
Similar to our bound, Equation~\ref{eqn: zou et al} also uses the effective dimension $k_1$ to characterize the variance term.
However, we emphasize that \citet{zou2021benign} derive the bound in a pretty different scenario from ours, which is one-pass SGD scenario.
During the one-pass SGD training, one uses a fresh data point to perform stochastic gradient descent in each epoch, and therefore they set $t = \Theta(n)$ by default.
As a comparison, we apply the standard full-batch gradient descent, and thus the time can be more flexible.
Besides, our results in Corollary~\ref{sqrt n} improve the bound in Equation~\eqref{eqn: zou et al} by dropping the second term.
We refer to the third and fourth example in Example \ref{Allexample} for a numerical comparison of the bounds\footnote{Due to the bias term in Theorem~\ref{main thm}, the overall excess risk bound cannot surpass the order $O(1/\sqrt{n})$, which leads to the cases that \citet{zou2021benign} outperforms our bound. However, we note that such differences come from the intrinsic property of GD and SGD, which may be unable to avoid in the GD regimes.}.

\subsection{Comparisons with Stability-Based Bounds}\label{app: stability}
In this section, we show that Theorem~\ref{main thm} gives provably better upper bounds than the stability-based method. We cite a result from~\citet{DBLP:journals/corr/abs-2106-06153}, which uses stability arguments to tackle overparameteried linear regression under similar assumptions.

\begin{theorem}[modified from Theorem 1 in~\citet{DBLP:journals/corr/abs-2106-06153}]\label{thm:stab}
Under the overparameterized linear regression settings, assume that $\|\mb x\|\le 1$, $|\varepsilon|\le V$, $w=\frac{\mb \theta^{*,\top}\mb x}{\sqrt{\mb \theta^{*,\top}\mb \Sigma\mb \theta^{*}}}$ is $\sigma_w^2$-subgaussian. Let $B_t=\sup_{\tau\in[t]}\|\mb \theta_t\|$. the following inequality holds with probability at least $1-\delta$:
\begin{equation}\label{eq:stab}
    R(\mb \theta_t)=\tilde{O}\left(\max\{1,\mb \theta^{*,\top}\mb \Sigma\mb \theta^{*}\sigma_w^2,(V+B_t)^2\}\sqrt{\frac{\log(4/\delta)}{n}}+ \frac{\|\mb\theta^*\|^2}{\lambda t}+\frac{\lambda t(V+B_t)^2}{n}\right).
\end{equation}
\end{theorem}

Theorem~\ref{thm:stab} applies the general stability-based results~\citep{DBLP:conf/icml/HardtRS16, DBLP:conf/colt/FeldmanV19} in the overparameterized linear regression setting, by replacing the bounded Lipschitz condition with the bounded domain condition. A fine-grained analysis~\citep{DBLP:conf/icml/LeiY20} may remove the bounded Lipschitz condition, but it additionally requires zero noise or decaying learning rate, which is different from our setting.
We omit the excess risk decomposition technique adopted in \citet{DBLP:journals/corr/abs-2106-06153} for presentation clarity.

Theorem~\ref{thm:stab} can not directly yield the stability argument in Theorem~\ref{thm:compat_main}, since obtaining a high probability bound of $B_t$ requires a delicate trajectory analysis and is a non-trivial task. Therefore, data-irrelevant methods such as stability-based bounds can not be directly applied to our setting. Even if one can replace $B_t$ in Equation~\ref{eq:stab} with its expectation that is easier to handle (this modification will require adding concentration-related terms, and make the bound in Equation~\ref{eq:stab} looser), we can still demonstrate that Theorem~\ref{main thm} is tighter than the corresponding stability-based analysis by providing a lower bound on $\ee [B_t^2]$, which will imply a lower bound on the righthand side of Equation~\ref{eq:stab}. 

\begin{theorem}\label{thm: stability ours}
Let Assumption~\ref{assumption 1},~\ref{assumption 2},~\ref{assump:fullRank} holds. Suppose $\lambda=O\left(\frac{1}{\sum_{i>0}\lambda_i}\right)$. Suppose the conditional variance of the noise $\varepsilon|\mb x$ is lower bounded by $\sigma_{\varepsilon}^2$. There exists constant $c$, such that with probability at least $1-n e^{-\frac{n}{c}}$, we have for $t=o(n)$,
\begin{equation}
\ee\|\mb\theta_t\|^2=\Omega\left(\frac{\lambda^2 t^2}{n}\left(\sum_{i>k_0}\lambda_i\right)\right)
\end{equation}
\end{theorem}

First we prove the following lemma,  bounding the number of large $\mu_i$.
\begin{lemma}\label{lem: large eigen number}
Suppose $t=o(n)$. Let $l$ denote the number of $\mu_i$, such that $\mu_i=\Omega\left(\frac{n}{t}\right)$. Then with probability at least $1-n e^{-\frac{n}{c}}$, we have $l=O(t)$.
\end{lemma}
\begin{proof}
According to Lemma~\ref{eigenvalue upperbound}, we know that with probability at least $1-n e^{-\frac{n}{c}}$, Equation~\ref{eq:eigenvalue upperbound} holds for all $0\le k\le n-1$. Conditioned on this, we have
\begin{equation}
    \frac{n}{t}l\lesssim\sum_{k=1}^l \mu_i\lesssim\sum_{k=1}^l(\sum_{i\ge k}\lambda_i+\lambda_{k}n)\lesssim (l+n)\sum_{i>0}\lambda_i\lesssim l+n.
\end{equation}
Since $t=o(n)$, we have $l=O(t)$ as claimed.
\end{proof}

We also need the result from \citet{bartlett2020benign}, which gives a lowerbound of $\mu_n$.
\begin{lemma}(Lemma 10 in \citet{bartlett2020benign})\label{lem: eigenvalue lowerbound}
For any $\sigma_x$, there exists a constant c, such that with probability at least $1-e^{-\frac{n}{c}}$ we have,
\begin{equation}\label{eq:eigenvalue lowerbound}
    \mu_{n}\ge c\left(\sum_{i>k_0}\lambda_i\right).
\end{equation}
\end{lemma}

We are now ready to prove Theorem~\ref{thm: stability ours}.
\begin{proof}
We begin with the calculation of $\|\mb \theta_t\|^2$. By Lemma~\ref{lem: theta_t formula}, the conditional unbiasedness of noise in Assumption~\ref{assumption 2} and the noise variance lower bound, we have
\begin{align}\label{eq: theta lower bound}
    \begin{aligned}
        \ee\|\mb\theta_t\|^2&=\left\|\left(\mb I-\frac{\lambda}{n}\mb X^{\top}\mb X\right)^t(\mb \theta_0-\mb X^{\dagger}\mb Y)+\mb X^{\dagger}Y\mb \right\|^2\\
        &=\ee\left\|\left(\mb I-\left(\mb I-\frac{\lambda}{n}\mb X^{\top}\mb X\right)^t\right)\mb X^{\dagger}\left(\mb X\mb \theta^*+\mb \varepsilon\right)\right\|^2\\
        &=\ee\left\|\left(\mb I-\left(\mb I-\frac{\lambda}{n}\mb X^{\top}\mb X\right)^t\right)\mb X^{\dagger}\mb X\mb \theta^*\right\|^2+\ee\left\|\left(\mb I-\left(\mb I-\frac{\lambda}{n}\mb X^{\top}\mb X\right)^t\right)\mb X^{\dagger} \mb \varepsilon\right\|^2\\
        &\ge \ee\left\|\left(\mb I-\left(\mb I-\frac{\lambda}{n}\mb X^{\top}\mb X\right)^t\right)\mb X^{\dagger} \mb \varepsilon\right\|^2\\
        &=\ee\Tr\left[\left(\mb I-\left(\mb I-\frac{\lambda}{n}\mb X^{\top}\mb X\right)^t\right)\mb X^{\dagger} \mb \varepsilon \mb \varepsilon^\top\mb  X^{\dagger,\top}\left(\mb I-\left(\mb I-\frac{\lambda}{n}\mb X^{\top}\mb X\right)^t\right)\right]\\
        &\ge \sigma_{\varepsilon}^2\ee\Tr\left[\left(\mb I-\left(\mb I-\frac{\lambda}{n}\mb X^{\top}\mb X\right)^t\right)\mb X^{\dagger}\mb X^{\dagger,\top}\left(\mb I-\left(\mb I-\frac{\lambda}{n}\mb X^{\top}\mb X\right)^t\right)\right]\\
        &=\sigma_{\varepsilon}^2\sum_{i=1}^n \frac{[1-(1-\frac{\lambda}{n}\mu_i)^t]^2}{\mu_i}.
    \end{aligned}
\end{align}

When $\mu_i=o\left(\frac{n}{t}\right)$, we have 
\begin{equation}
    1-(1-\frac{\lambda}{n}\mu_i)^t=1-1+\frac{\lambda}{n}\mu_it+O\left(\left(\frac{\lambda}{n}\mu_i t\right)^2\right)=\Theta\left(\frac{\lambda}{n}\mu_it\right).
\end{equation}
Plugging it into Equation~\ref{eq: theta lower bound} and then use Lemma~\ref{lem: large eigen number},~\ref{lem: eigenvalue lowerbound}, we know that under the high probability event in Lemma ~\ref{lem: large eigen number} and ~\ref{lem: eigenvalue lowerbound},
\begin{equation}
    \ee\|\mb\theta_t\|^2=\Omega\left((n-l)\frac{\lambda^2}{n^2}\mu_n t^2\right)=\Omega\left(\frac{\lambda^2}{n}\mu_n t^2\right)=\Omega\left(\frac{\lambda^2 t^2}{n}\left(\sum_{i>k_0}\lambda_i\right)\right)
\end{equation}
\end{proof}

Therefore, the stability-based bound, i.e., the right hand side of Equation~\ref{eq:stab}, can be lower bounded in expectation as
$\Omega\left(\frac{\lambda^3 t^3}{n^2}\sum_{i>k_0}\lambda_i\right)$. This implies that the stability-based bound is vacuous when $t=\Omega\left(\frac{n^{\frac{2}{3}}\left(\sum_{i>k_0}\lambda_i\right)^{-\frac{1}{3}}}{\lambda}\right)$. Thus, stability-based methods will provably yield smaller compatibility region than $\left(\omega\left(\frac{1}{\lambda}\right),o\left(\frac{n}{\lambda}\right)\right)$ in Theorem~\ref{thm:compat_main} when $\sum_{i>k_0}\lambda_i$ is not very small, as demonstrated in the examples below.

\begin{example}\label{examp: stability}
Let Assumption~\ref{assumption 1},~\ref{assumption 2},~\ref{assump:fullRank} holds. Assume without loss of generality that $\lambda=\Theta(1)$.
We have the following examples:
\begin{enumerate}
    \item \textbf{\emph{(Inverse Polynominal).}} If the spectrum of $\mb \Sigma$ satisfies $$\lambda_k=\frac{1}{k^\alpha},$$ for some $\alpha>1$, we derive that $k_0=\Theta(n)$, $\sum_{i>k_0}\lambda_i=\Theta(\frac{1}{n^{\alpha-1}})$. Therefore, the stability bound in Theorem~\ref{thm:stab} is vacuous when 
    $$t=\Omega\left(n^{\frac{\alpha+1}{3}}\right),$$
    which is outperformed by the region in Theorem~\ref{main thm} when $\alpha<2$.
    
    \item \textbf{\emph{(Inverse Log-Polynominal).}} If the spectrum of $\mb \Sigma$ satisfies $$\lambda_k=\frac{1}{k \log^\beta (k+1)},$$ for some $\beta>1$
    , we derive that $k_0=\Theta\left(\frac{n}{\log n}\right)$, $\sum_{i>k_0}\lambda_i=\tilde{\Theta}(1)$. Therefore, the stability bound in Theorem~\ref{thm:stab} is vacuous when 
    $$t=\tilde{\Omega}\left(n^{\frac{2}{3}}\right),$$
    which is outperformed by the region in Theorem~\ref{main thm}.

    \item \textbf{\emph{(Constant).}} If the spectrum of $\mb \Sigma$ satisfies $$\lambda_k=\frac{1}{n^{1+\varepsilon}},1\le k \le n^{1+\varepsilon},$$
    for some $\varepsilon>0$, we derive that $k_0=0$, $\sum_{i>k_0}\lambda_i=1$. Therefore, the stability bound in Theorem~\ref{thm:stab} is vacuous when 
    $$t=\Omega\left(n^{\frac{2}{3}}\right),$$
    which is outperformed by the region in Theorem~\ref{main thm}.
    
    \item \textbf{\emph{(Piecewise Constant).}} If the spectrum of $\mb \Sigma$ satisfies\\ $$\lambda_k=\begin{cases}\frac{1}{s}&1\le k\le s,\\ \frac{1}{d-s}& s+1\le k\le d,\end{cases}$$ where $s= n^r,d=n^q,0<r\le 1,q\ge 1$, we derive that $k_0=n^r$, $\sum_{i>k_0}\lambda_i=1$. Therefore, the stability bound in Theorem~\ref{thm:stab} is vacuous when 
    $$t=\Omega\left(n^{\frac{2}{3}}\right),$$
    which is outperformed by the region in Theorem~\ref{main thm}.
\end{enumerate}
\end{example}

\subsection{Comparisons with Uniform Convergence Bounds}\label{app: uniform}
We first state a standard bound on the Rademacher complexity of linear models.

\begin{theorem}[Theorem in~\citet{DBLP:books/daglib/0034861}]
Let $S\subseteq \{\mb x: \|x\|_2\le r\}$ be a sample of size $n$ and let $\cH=\{x\mapsto \left<w,x\right>:\|w\|_2\le \Lambda\}$. Then, the empirical Rademacher complexity of $\cH$ can be bounded as follows:
\begin{equation}
    \hat{\cR}_S(\cH)\le \sqrt{\frac{r^2\Lambda^2}{n}}.
\end{equation}
\end{theorem}
Furthermore, Talagrand's Lemma (See Lemma 5.7 in ~\citet{DBLP:books/daglib/0034861}) indicates that 
\begin{equation}
    \hat{\cR}_S(l\circ\cH)\le L \hat{\cR}_S(\cH)=\frac{\Theta(\Lambda^2)}{\sqrt{n}},
\end{equation}
where $L=\Theta(\Lambda)$ is the Lipschitz coefficient of the square loss function $l$ in our setting. Therefore, the Rademacher generalization bound is vacuous when $\Lambda=\Omega(n^\frac{1}{4})$. By Theorem~\ref{thm: stability ours}, we know that $\ee\|\mb\theta_t\|^2=\Omega(n^\frac{1}{2})$ when $t=\Omega\left(\frac{n^{\frac{3}{4}}}{\lambda\left(\sum_{i>k_0}\lambda_i\right)^{\frac{1}{2}}}\right)$. A similar comparison as in Example~\ref{examp: stability} can demonstrate that uniform stability arguments will provably yield smaller compatibility region than that in Theorem~\ref{main thm} for example distributions.

\subsection{Comparison with Previous Works on Early Stopping}
\label{app: early stopping}
A line of works focuses on deriving the excess risk guarantee of linear regression or kernel regression with early stopping (stochastic) gradient descent.
We refer to Section~\ref{sec: related works} for details.
Here we compare our results with some most relevant works, including~\citep{yao2007early,DBLP:journals/jmlr/LinR17,DBLP:conf/nips/Pillaud-VivienR18}. 

\textbf{Comparison with \citet{yao2007early}.} \citet{yao2007early} study kernel regression with early stopping gradient descent. Their approaches are different from ours in the following aspects. 

Firstly, the assumptions used in the two approaches are different, due to different goals and techniques. \citet{yao2007early} assume that the input feature and data noise have bounded norm (see Section 2.1 in \citet{yao2007early}), while we require that the input feature is subgaussian with independent entries.  

Furthermore, although \citet{yao2007early} obtain a minimax bound in terms of the convergence rate, it is suboptimal in terms of compatibility region. Specifically, The results in our paper show a region like $(0, n)$ while the techniques \citet{yao2007early} can only lead to a region like $(0, \sqrt{n})$. See Proof of the Main Theorem in section 2 in \citet{yao2007early} for details. Such differences come from different goals of the two approaches, where \citet{yao2007early} focus on providing the optimal early-stopping time while we focus on providing a larger time region in which the loss is consistent.

\textbf{Comparison with \citet{DBLP:journals/jmlr/LinR17}.} 
\citet{DBLP:journals/jmlr/LinR17} study stochastic gradient descent with arbitrary batchsize, which is reduced to full batch gradient descent when setting the batchsize to sample size $n$. Their results are different from ours, since they require the boundness assumption, and focus more on the optimal early stopping time rather than the largest compatibility region, in the same spirit of~\citet{yao2007early}. Specifically,~\citet{DBLP:journals/jmlr/LinR17} derive a region like $(0,n^{\frac{\zeta+1}{2\zeta+\gamma}})$, where $\zeta$ and $\gamma$ are problem dependent constants (See Theorem 1 in~\citet{DBLP:journals/jmlr/LinR17} for details). The following examples demonstrate that this paper's results yield larger regions for a wide range of distribution classes. 

\begin{example}
\textbf{\emph{(Inverse Polynominal).}} If the spectrum of $\mb \Sigma$ satisfies $$\lambda_k=\frac{1}{k^\alpha},$$ for some $\alpha>1$. For this distribution, we have $\zeta=\frac{1}{2}$, $\gamma=\frac{1}{\alpha}$, and their region is $(0,n^{\frac{3\alpha}{2\alpha+1}})$, which is smaller than $(0,n^{\frac{3\alpha+1}{2\alpha+1}})$ given in Example~\ref{eg: non-constant Example}. 
\end{example}

\begin{example}
\textbf{\emph{(Inverse Log-Polynominal).}} If the spectrum of $\mb \Sigma$ satisfies $$\lambda_k=\frac{1}{k \log^\beta (k+1)},$$ for some $\beta>1$. 
For this distribution, we have $\zeta=\frac{1}{2}$, $\gamma=1$, and their region is $(0,n^{\frac{3}{4}})$, which is smaller than $(0,n)$ given Corollary~\ref{sqrt n}. 
\end{example}

\subsection{Calculations in Example~\ref{Allexample}}
We calculate the quantities $r(\Sigma),k_0,k_1,k_2$ for the example distributions in \ref{Allexample}. The results validate that $k_1$ is typically a much smaller quantity than $k_0$.
\begin{enumerate}
\item \textbf{Calculations for $\lambda_k=\frac{1}{k^\alpha},\alpha>1$.}

Define $r_k(\mb \Sigma)=\frac{\sum_{i>k}\lambda_i}{\lambda_{k+1}}$ as in \citet{bartlett2020benign}. Since $\sum_{i>k}\frac{1}{i^\alpha}=\Theta(\frac{1}{k^{\alpha-1}})$, we have $r_k(\mb \Sigma)=\Theta\left(\frac{\frac{1}{k^{\alpha-1}}}{\frac{1}{k^{\alpha}}}\right)=\Theta(k)$. Hence, $k_0=\Theta(n)$
\footnote{The calculations for $k_0, k_1$ and $k_2$ in this section only apply when $n$ is sufficiently large.}, and the conditions of theorem \ref{main thm} is satisfied.

As $\sum_{i>0}\lambda_i<\infty$, By its definition we know that $k_1$ is the smallest $l$ such that $\lambda_{l+1}=O(\frac{1}{n})$. Therefore, $k_1=\Theta(n^{\frac{1}{\alpha}})$. 

\item \textbf{Calculations for $\lambda_k=\frac{1}{k \log^\beta (k+1)},\beta>1$.}

$\sum_{i>k}\frac{1}{i\log^\beta (i+1)}=\Theta(\int_{k}^{\infty}\frac{1}{x\log^\beta x}dx)=\Theta(\frac{1}{\log^{\beta-1}k})$, which implies $r_k(\mb \Sigma)=k\log k$. Solving $k_0\log k_0\ge \Theta(n)$, we have $k_0=\Theta(\frac{n}{\log n})$.

By the definition of $k_1$, we know that $k_1$ is the smallest $l$ such that $l\log^\beta (l+1)\ge \Theta(n)$. Therefore, $k_1=\Theta(\frac{n}{\log^\beta n})$. 

\item \textbf{Calculations for $\lambda_i=\frac{1}{n^{1+\varepsilon}},1\le i \le n^{1+\varepsilon},\varepsilon>0$.}

Since $r_0(\mb \Sigma)=n^{1+\varepsilon}$, we have $k_0=0$. By the definition of $k_1$, we also have $k_1=0$.

\item \textbf{Calculations for $\lambda=\begin{cases}\frac{1}{s}&1\le k\le s\\ \frac{1}{d-s}& s+1\le k\le d\end{cases},s= n^r,d=n^q,0<r\le 1,q\ge 1$.}

For $0\le k<s$, $r_k(\mb \Sigma)=\Theta(\frac{1}{\frac{1}{s}})=\Theta(n^r)=o(n)$, while $r_s(\mb \Sigma)=\frac{1}{\frac{1}{d-s}}=\Theta(n^q)=\omega(n)$. Therefore, $k_0=s=n^r$.

Similarly, noting that $\lambda_k=\frac{1}{n^r}=\omega(n)$ for $0\le k<s$ and $\lambda_s=\Theta(\frac{1}{n^d})=o(n)$, we know that $k_1=s=n^r$. 
\end{enumerate}

\subsection{Calculations for \texorpdfstring{$\lambda_k=1/k^\alpha,\alpha>1$}{Lg} in Section \ref{further examples}}
Set $c(t,n)=\frac{1}{n^\beta}$, where $\beta>0$ will be chosen later. First we calculate $k_2$ under this choice of $c(t,n)$. Note that $\sum_{i>k}\frac{1}{i^\alpha}=\Theta\left(\frac{1}{k^{\alpha-1}}\right)$. Therefore, $k_2$ is the smallest $k$ such that $\frac{1}{k^{\alpha-1}}+\frac{n}{k^\alpha}=O(\frac{1}{n^\beta})$. For the bound on $V(\theta_t)$ to be consistent, we need $k_2=o(n)$. Hence, $\frac{1}{k^{\alpha-1}}=O(\frac{n}{k^\alpha})$, which implies $k_2=n^{\frac{\beta+1}{\alpha}}$.

Plugging the value of $c(t,n)$ and $k_2$ into our bound, we have 
$$V(\theta_t)=O\left(n^{(\frac{1}{\alpha}-1)+(\frac{1}{\alpha}+1)\beta}+n^{2\tau-\beta-2}\right)$$
which attains its minimum $\Theta(n^{\frac{2\alpha\tau-3\alpha+2\tau-1}{2\alpha+1}})$ at $\beta=\Theta\left(\frac{2\alpha\tau-\alpha-1}{2\alpha+1}\right)$.

For $V(\theta)=O(1)$, we need $\tau\le \frac{3\alpha+1}{2\alpha+2}$. For $\beta\ge 0$, we need $\tau\ge\frac{\alpha+1}{2\alpha}$. Putting them together gives the range of $t$ in which the above calculation applies.

\subsection{Discussions on \texorpdfstring{$\cD_n$}{Lg}}
In this paper, the distribution $\cD$ is regarded as a sequence of distribution $\{ \cD_n\}$ which may dependent on sample size $n$. 
The phenomenon comes from overparameterization and asymptotic requirements.  
In the definition of compatibility, we require $n \to \infty$.
In this case, overparameterization requires that the dimension $p$ (if finite) cannot be independent of $n$ since $n \to \infty$ would break overparameterization. 
Therefore, the covariance $\Sigma$ also has $n$-dependency, since $\Sigma$ is closely related to $p$.

Several points are worth mentioning: (1) Similar arguments generally appear in related works, for example, in~\citet{bartlett2020benign} when discussing the definition of benign covariance. (2) One can avoid such issues by considering an infinite dimensional feature space. This is why we discuss the special case $p = \infty$ in Theorem~\ref{thm:compat_main}. (3) If $p$ is a fixed finite constant that does not alter with $n$, the problem becomes underparameterized and thus trivial to get a consistent generalization bound via standard concentration inequalities.

%% file: text/appendix_experiment.tex
\section{Additional Experiment Results}\label{section additional experiments}

\subsection{Details for Linear Regression Experiments}
In this section, we provide the experiment details for linear regression experiments  and present additional empirical results.

The linear regression experiment in Figure~\ref{fig:intro} follows the setting described in section \ref{sec: exp}.
Although the final iterate does not interpolate the training data, the results suffice to demonstrate the gap between the early-stopping and final-iterate excess risk. The training plot for different covariances are given in Figure~\ref{fig:linear_appen}.

\begin{figure*}[t]  
\centering
\subfigure[$\lambda_i=\frac{1}{i}$]{
\begin{minipage}{0.29\linewidth}
\centerline{\includegraphics[width=1\textwidth]{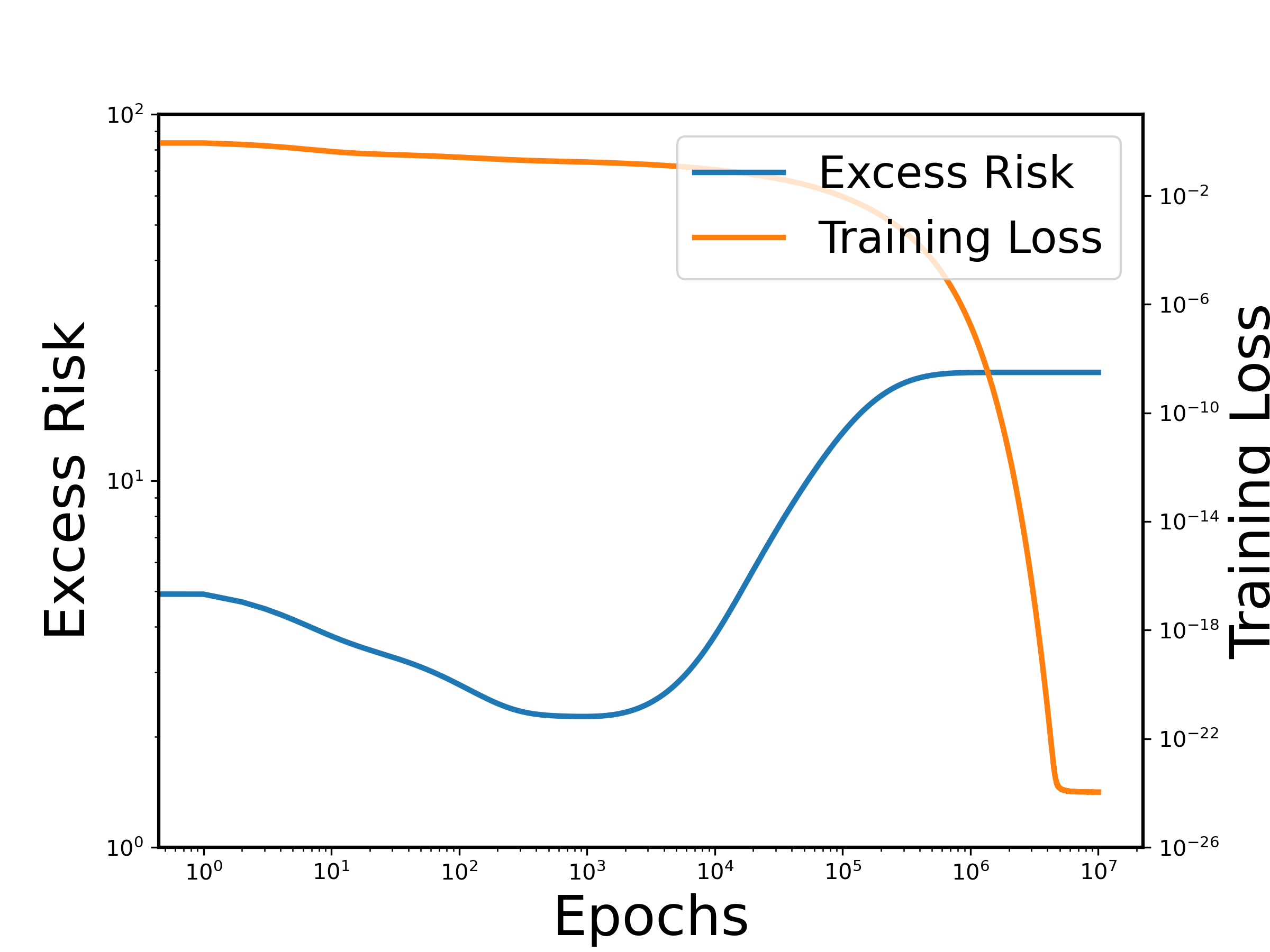}}
\end{minipage}
}
\quad
\subfigure[$\lambda_i=\frac{1}{i^2}$]{
\begin{minipage}{0.29\linewidth}
\centerline{\includegraphics[width=1\textwidth]{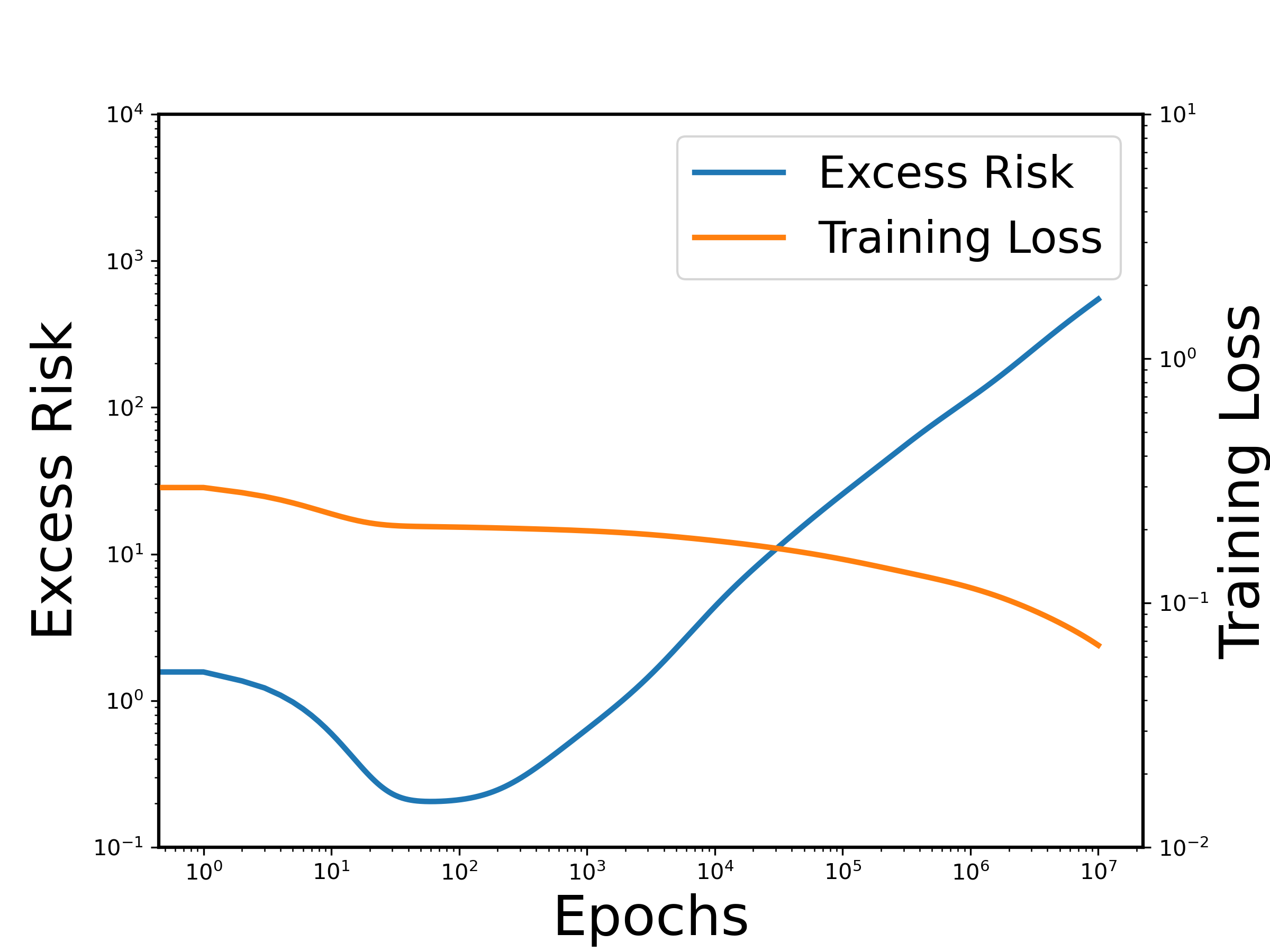}}
\end{minipage}}
\quad
\subfigure[$\lambda_i=\frac{1}{i^3}$]{
\begin{minipage}{0.29\linewidth}
\centerline{\includegraphics[width=1\textwidth]{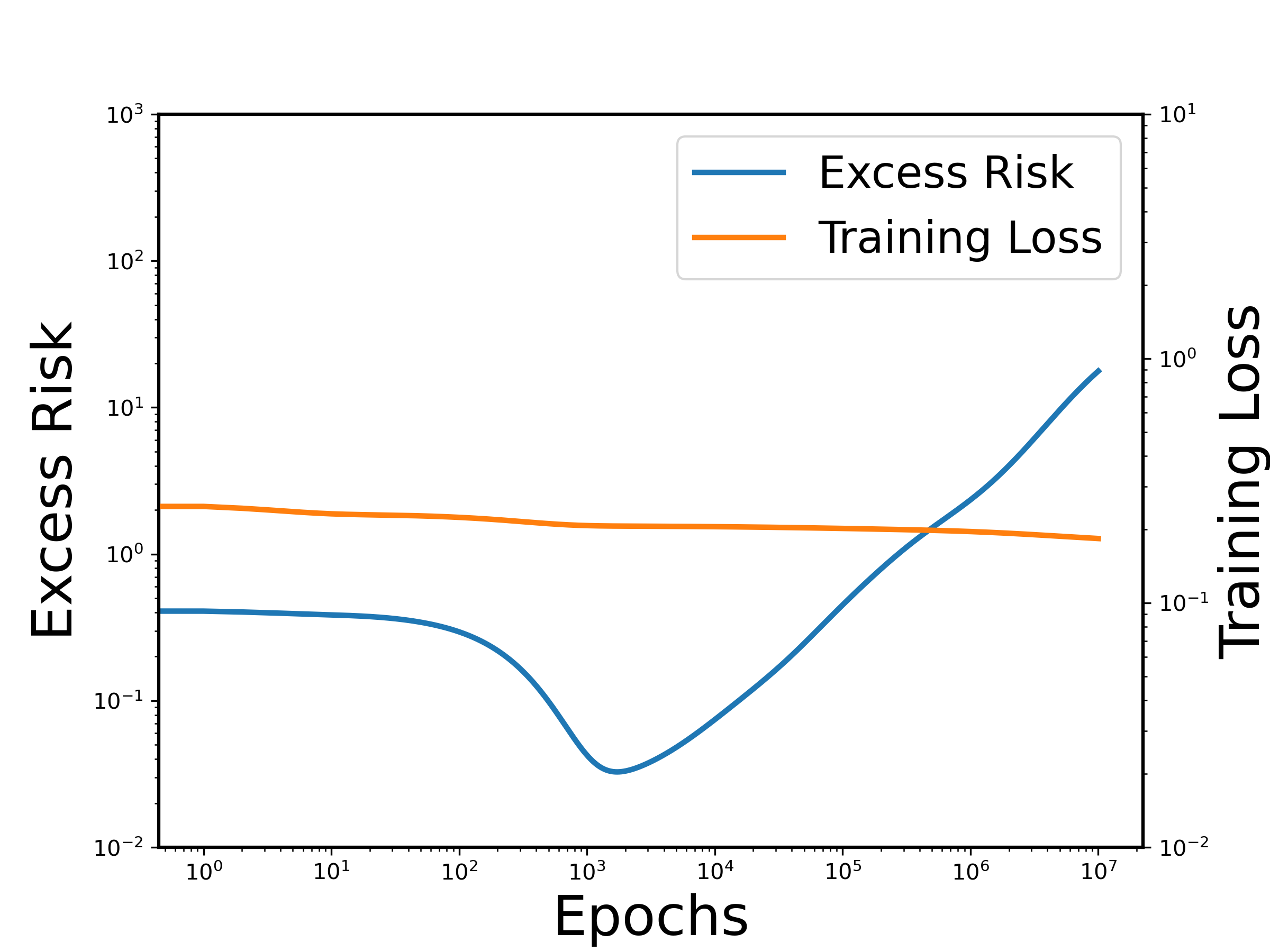}}
\end{minipage}
}
\\
\subfigure[$\lambda_i=\frac{1}{i\log(i+1)}$]{
\begin{minipage}{0.29\linewidth}
\centerline{\includegraphics[width=1\textwidth]{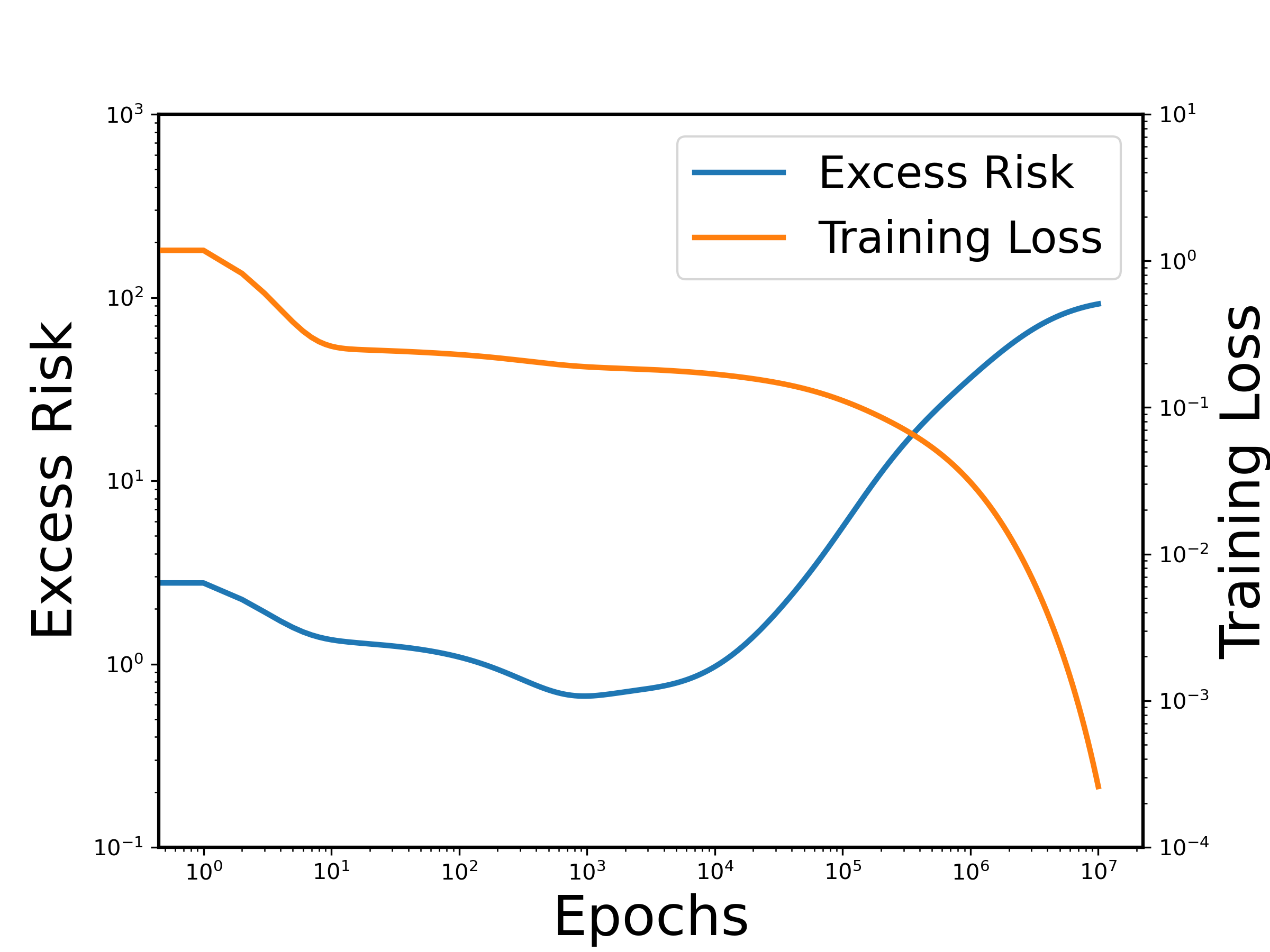}}
\end{minipage}}
\quad
\subfigure[$\lambda_i=\frac{1}{i\log^2(i+1)}$ ]{
\begin{minipage}{0.29\linewidth}
\centerline{\includegraphics[width=1\textwidth]{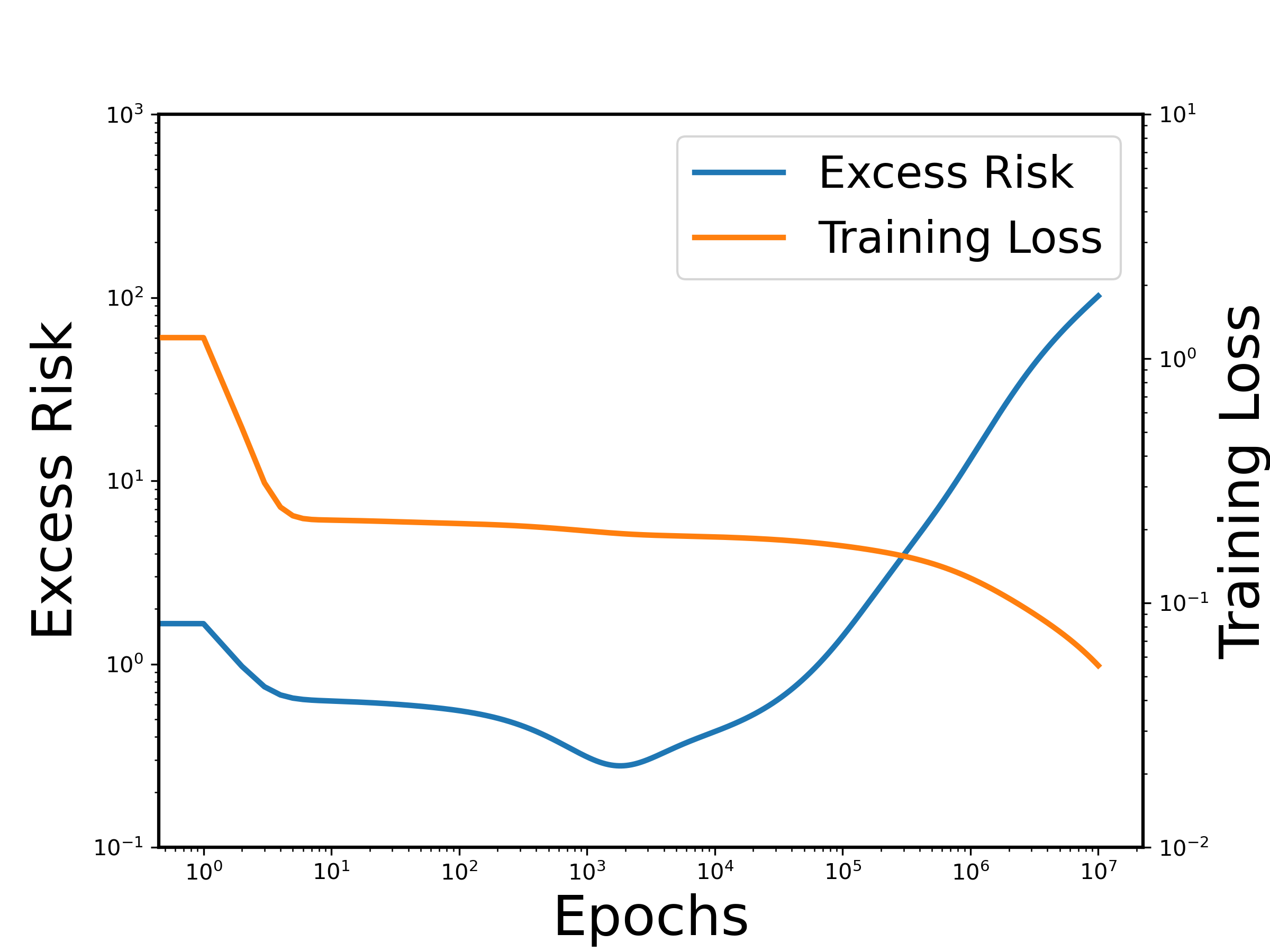}}
\end{minipage}
}
\quad
\subfigure[$\lambda_i=\frac{1}{i\log^3(i+1)}$]{
\begin{minipage}{0.29\linewidth}
\centerline{\includegraphics[width=1\textwidth]{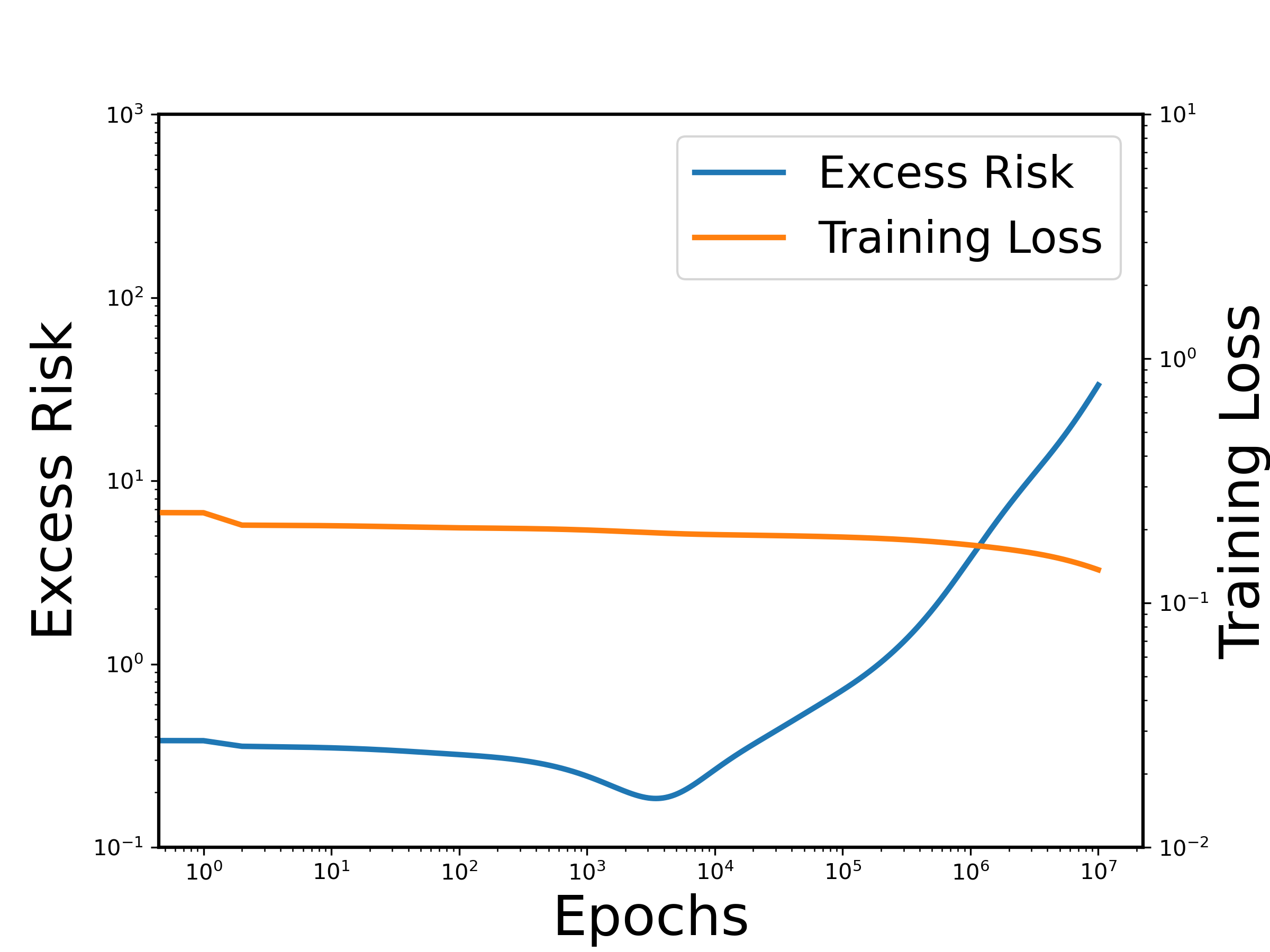}}
\end{minipage}}

\caption{\textbf{The training plot for overparameterized linear regression with different covariances using GD.}}
\label{fig:linear_appen}
\end{figure*}

Next, we provide the experiment results of sample sizes $n=50$, $n=200$ and feature dimensions $p=500$, $p=2000$. The settings are the same as described above, except for the sample size. The optimal excess risk and min-norm excess risk are provided in Table~\ref{sample-table 50},~\ref{sample-table 200},~\ref{sample-table 100 2000} and~\ref{sample-table 100 500}. The tables indicate that the two observations stated above hold for different sample size $n$.

\begin{table}[t]
\caption{\textbf{The effective dimension $k_1$, the optimal early stopping excess risk and min-norm excess risk for different feature distributions, with sample size $n=50$, $p=1000$.}
We calculate the 95\% confidence interval for the excess risk.}
\label{sample-table 50}
\vskip 0.15in
\begin{center}
\begin{small}
\begin{sc}
\begin{tabular}{lcccr}
\toprule
Distributions & $k_1$ & Optimal Excess Risk & Min-norm excess risk \\
\midrule
$\lambda_i=\frac{1}{i}$& $\Theta(n)$& $2.515 \pm 0.0104$ & $12.632 \pm 0.1602 $ \\
$\lambda_i=\frac{1}{i^2}$ &$\Theta(n^{\frac{1}{2}})$ &$0.269 \pm 0.0053$ & $50.494 \pm 0.9378 $ \\
$\lambda_i=\frac{1}{i^3}$ &$\Theta(n^{\frac{1}{3}})$ &$0.083 \pm 0.0011$& $13.208 \pm 0.4556 $ \\
$\lambda_i=\frac{1}{i\log(i+1)}$ & $\Theta(\frac{n}{\log n})$ & $0.808 \pm 0.0090$& $46.706\pm 0.6983 $ \\
$\lambda_i=\frac{1}{i\log^2(i+1)}$ & $\Theta(\frac{n}{\log^2 n})$ & $0.381 \pm 0.0076$& $74.423 \pm 1.1472 $  \\
$\lambda_i=\frac{1}{i\log^3(i+1)}$&$\Theta(\frac{n}{\log^3 n})$ & $0.233 \pm 0.0052$ & $43.045 \pm 0.8347 $ \\

\bottomrule
\end{tabular}
\end{sc}
\end{small}
\end{center}
\vskip -0.1in
\end{table}

\begin{table}[t]
\caption{
\textbf{The effective dimension $k_1$, the optimal early stopping excess risk and min-norm excess risk for different feature distributions, with sample size $n=200$ , $p=1000$.}
We calculate the 95\% confidence interval for the excess risk.}
\label{sample-table 200}
\vskip 0.15in
\begin{center}
\begin{small}
\begin{sc}
\begin{tabular}{lcccr}
\toprule
Distributions & $k_1$ & Optimal Excess Risk& Min-norm excess risk \\
\midrule
$\lambda_i=\frac{1}{i}$& $\Theta(n)$& $2.173 \pm 0.0065$& $52.364 \pm 0.4009 $ \\
$\lambda_i=\frac{1}{i^2}$ &$\Theta(n^{\frac{1}{2}})$ &$0.161\pm 0.0039$&$36.855 \pm 0.4833 $ \\
$\lambda_i=\frac{1}{i^3}$ &$\Theta(n^{\frac{1}{3}})$ &$0.068\pm 0.0012$&$8.1990 \pm 0.2313 $  \\
$\lambda_i=\frac{1}{i\log(i+1)}$& $\Theta(\frac{n}{\log n})$& $0.628\pm 0.0034$ &$152.70 \pm 1.1073 $ \\
$\lambda_i=\frac{1}{i\log^2(i+1)}$&$\Theta(\frac{n}{\log^2 n})$ & $0.241 \pm 0.0036$&$83.550 \pm 0.7596 $ \\
$\lambda_i=\frac{1}{i\log^3(i+1)}$&$\Theta(\frac{n}{\log^3 n})$ & $0.146\pm 0.0108 $ &$33.469 \pm 0.4540 $\\

\bottomrule
\end{tabular}
\end{sc}
\end{small}
\end{center}
\vskip -0.1in
\end{table}

\begin{table}[t]
\caption{
\textbf{The effective dimension $k_1$, the optimal early stopping excess risk and min-norm excess risk for different feature distributions, with sample size $n=100$ , $p=500$.}
We calculate the 95\% confidence interval for the excess risk.}
\label{sample-table 100 500}
\vskip 0.15in
\begin{center}
\begin{small}
\begin{sc}
\begin{tabular}{lcccr}
\toprule
Distributions & $k_1$ & Optimal Excess Risk& Min-norm excess risk \\
\midrule
$\lambda_i=\frac{1}{i}$& $\Theta(n)$& $1.997 \pm 0.0876$& $27.360 \pm 0.3019 $ \\
$\lambda_i=\frac{1}{i^2}$ &$\Theta(n^{\frac{1}{2}})$ &$0.211\pm 0.0050$&$43.531\pm 0.7025 $ \\
$\lambda_i=\frac{1}{i^3}$ &$\Theta(n^{\frac{1}{3}})$ &$0.076\pm 0.0011$&$10.062 \pm 0.3022 $  \\
$\lambda_i=\frac{1}{i\log(i+1)}$& $\Theta(\frac{n}{\log n})$& $0.645\pm 0.0056$ &$96.465 \pm 1.0594 $ \\
$\lambda_i=\frac{1}{i\log^2(i+1)}$&$\Theta(\frac{n}{\log^2 n})$ & $0.289 \pm 0.0055$&$83.694\pm 0.9827 $ \\
$\lambda_i=\frac{1}{i\log^3(i+1)}$&$\Theta(\frac{n}{\log^3 n})$ & $0.181\pm 0.0045 $ &$38.090 \pm 0.6378 $\\

\bottomrule
\end{tabular}
\end{sc}
\end{small}
\end{center}
\vskip -0.1in
\end{table}

\begin{table}[t]
\caption{
\textbf{The effective dimension $k_1$, the optimal early stopping excess risk and min-norm excess risk for different feature distributions, with sample size $n=100$ , $p=2000$.}
We calculate the 95\% confidence interval for the excess risk.}
\label{sample-table 100 2000}
\vskip 0.15in
\begin{center}
\begin{small}
\begin{sc}
\begin{tabular}{lcccr}
\toprule
Distributions & $k_1$ & Optimal Excess Risk& Min-norm excess risk \\
\midrule
$\lambda_i=\frac{1}{i}$& $\Theta(n)$& $2.662 \pm 0.0066$& $23.111 \pm 0.2278 $ \\
$\lambda_i=\frac{1}{i^2}$ &$\Theta(n^{\frac{1}{2}})$ &$0.219\pm 0.0050$&$43.130\pm 0.6421 $ \\
$\lambda_i=\frac{1}{i^3}$ &$\Theta(n^{\frac{1}{3}})$ &$0.077\pm 0.0010$&$10.031 \pm 0.2942 $  \\
$\lambda_i=\frac{1}{i\log(i+1)}$& $\Theta(\frac{n}{\log n})$& $0.749\pm 0.0055$ &$88.744 \pm 0.9414 $ \\
$\lambda_i=\frac{1}{i\log^2(i+1)}$&$\Theta(\frac{n}{\log^2 n})$ & $0.312 \pm 0.0057$&$82.859 \pm 0.9394 $ \\
$\lambda_i=\frac{1}{i\log^3(i+1)}$&$\Theta(\frac{n}{\log^3 n})$ & $0.190\pm 0.0047 $ &$37.782 \pm 0.5945 $\\

\bottomrule
\end{tabular}
\end{sc}
\end{small}
\end{center}
\vskip -0.1in
\end{table}

\subsection{Details for MNIST Experiments}
In this section, we provide the experiment details and additional results in MNIST dataset.

The MNIST experiment details are described below. We create a noisy version of MNIST with label noise rate $20\%$, i.e. randomly perturbing the label with probability $20\%$ for each training data, to simulate the label noise which is common in real datasets, e.g ImageNet~\citep{DBLP:conf/eccv/StockC18,DBLP:conf/icml/ShankarRMFRS20,DBLP:conf/cvpr/YunOHHCC21}. We do not inject noise into the test data. 

We choose a standard four layer convolutional neural network as the classifier. We use a vanilla SGD optimizer without momentum or weight decay. The initial learning rate is set to $0.5$. Learning rate is decayed by 0.98 every epoch. Each model is trained for 300 epochs. The training batch size is set to 1024, and the test batch size is set to 1000. We choose the standard cross entropy loss as the loss function.

We provide the plot for different levels of label noise in Figure~\ref{fig:mnist_appen}. We present the corresponding test error of the best early stopping iterate and the final iterate in Table \ref{tab: mnist error}. Since the theoretical part of this paper focuses on GD, we also provide a corresponding plot of GD training in Figure~\ref{fig:mnist_gd} for completeness.

\begin{figure*}[t]  
\centering
\subfigure[0\% label noise]{
\begin{minipage}{0.29\linewidth}
\centerline{\includegraphics[width=1\textwidth]{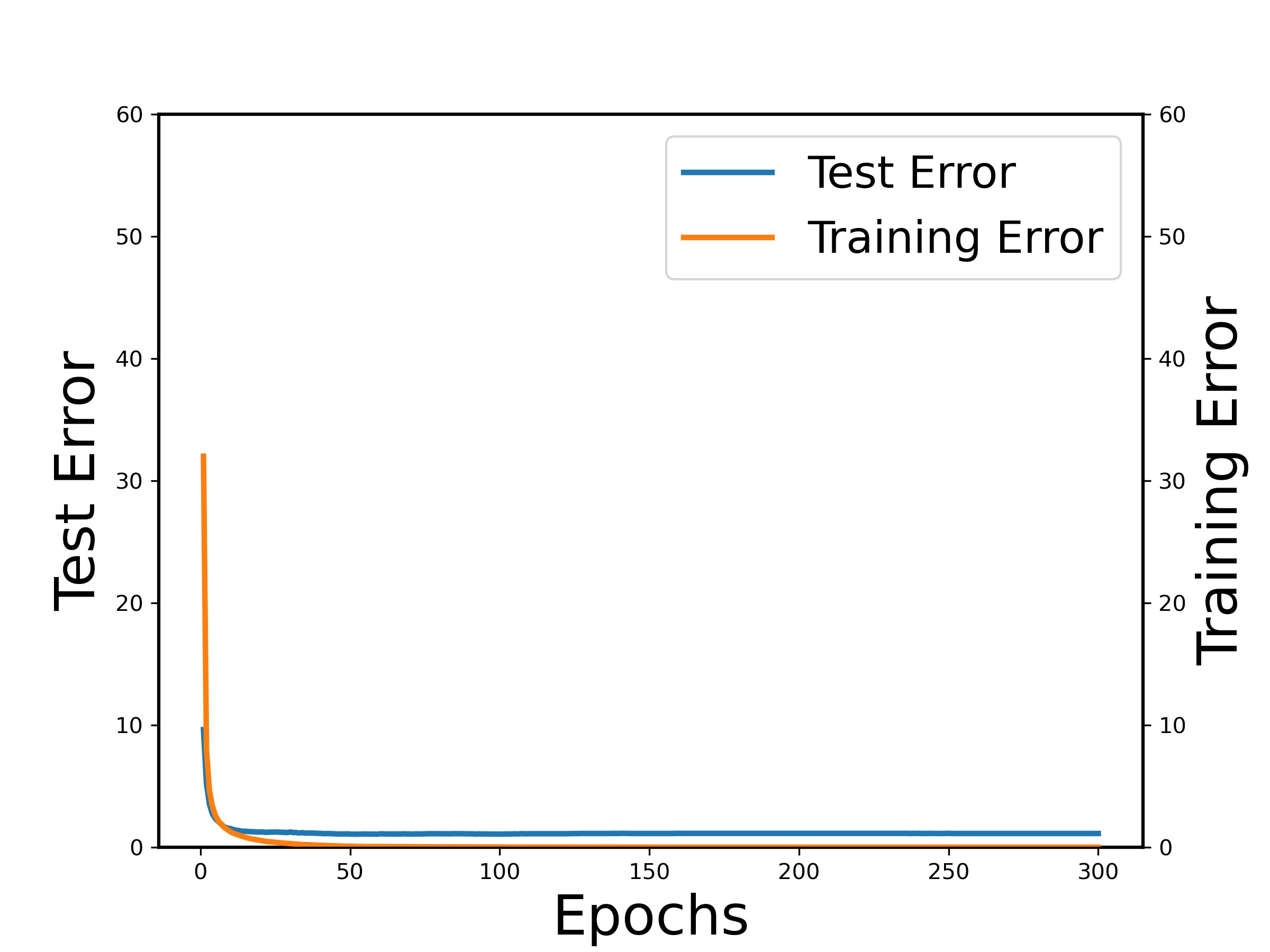}}
\end{minipage}
}
\quad
\subfigure[10\% label noise]{
\begin{minipage}{0.29\linewidth}
\centerline{\includegraphics[width=1\textwidth]{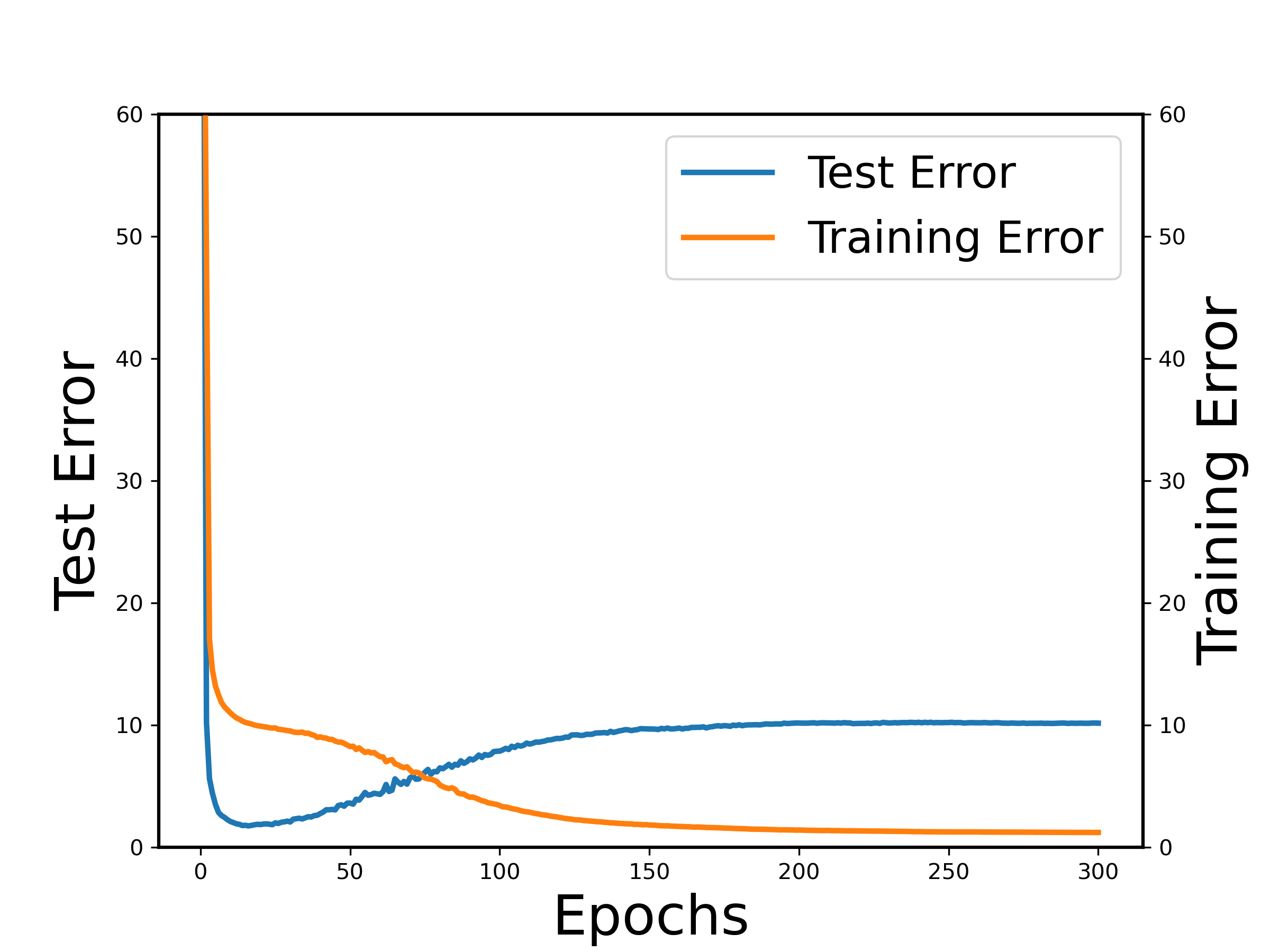}}
\end{minipage}}
\quad
\subfigure[20\% label noise]{
\begin{minipage}{0.29\linewidth}
\centerline{\includegraphics[width=1\textwidth]{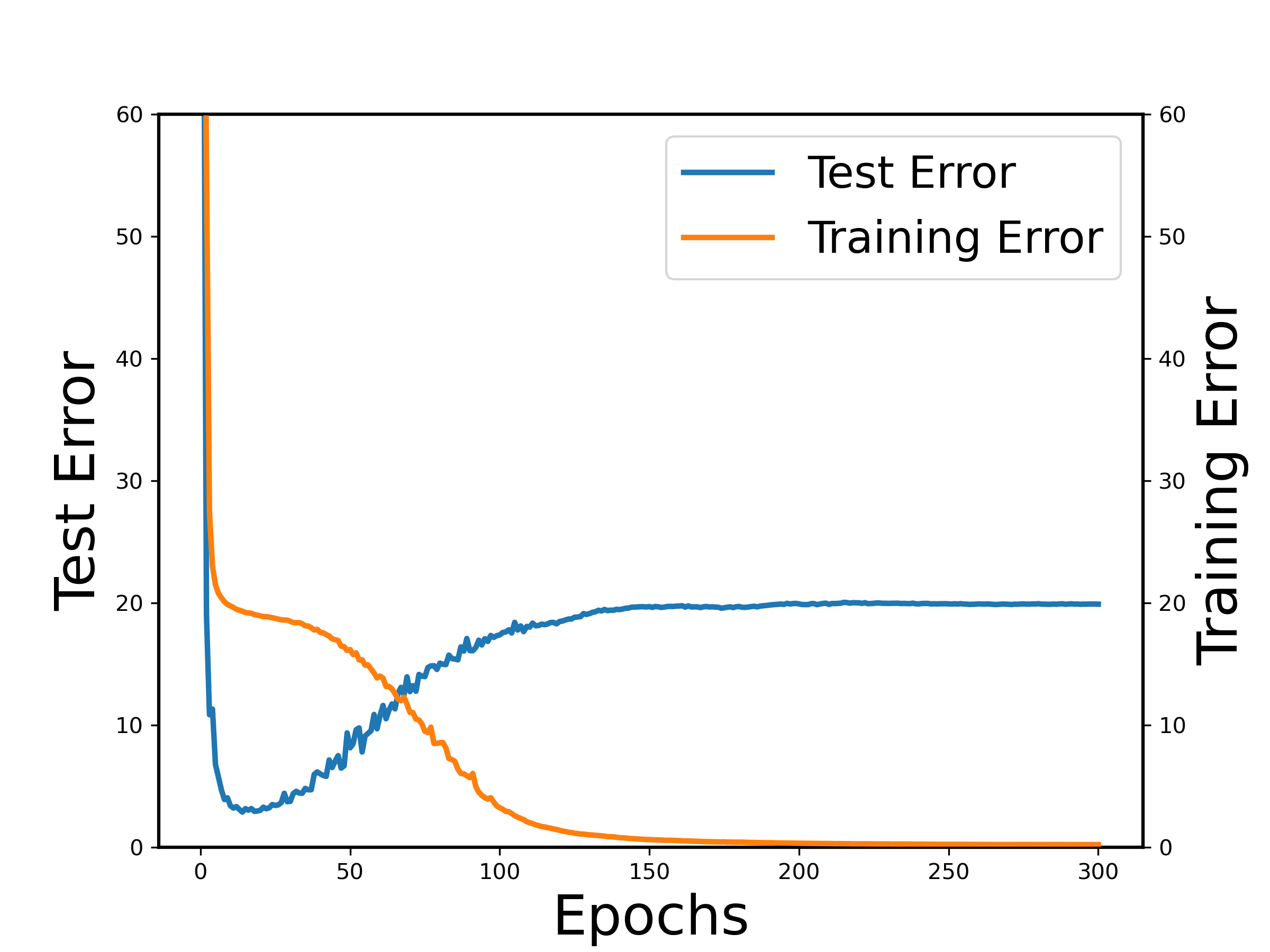}}
\end{minipage}
}
\\
\subfigure[30\% label noise]{
\begin{minipage}{0.29\linewidth}
\centerline{\includegraphics[width=1\textwidth]{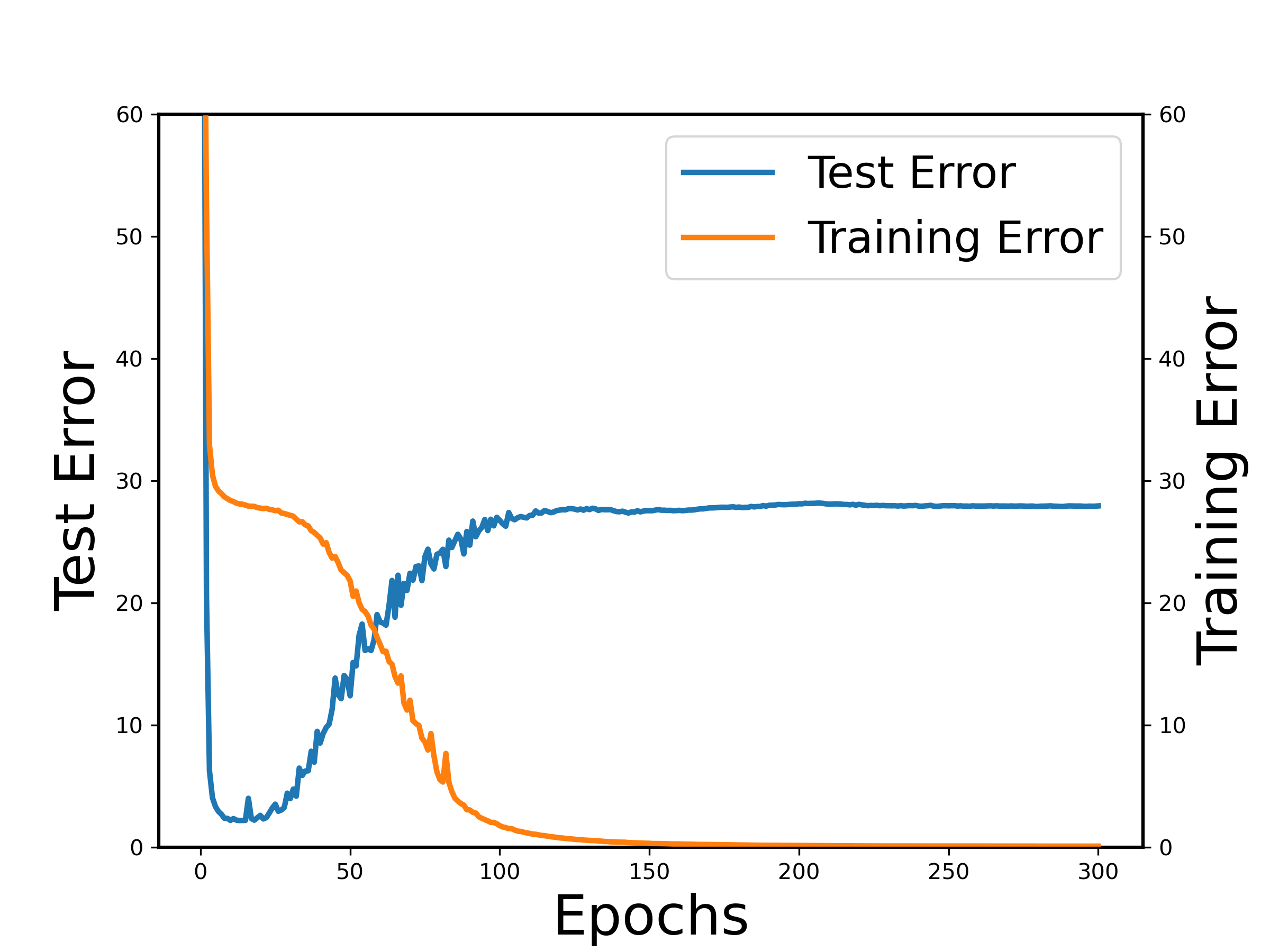}}
\end{minipage}}
\quad
\subfigure[40\% label noise]{
\begin{minipage}{0.29\linewidth}
\centerline{\includegraphics[width=1\textwidth]{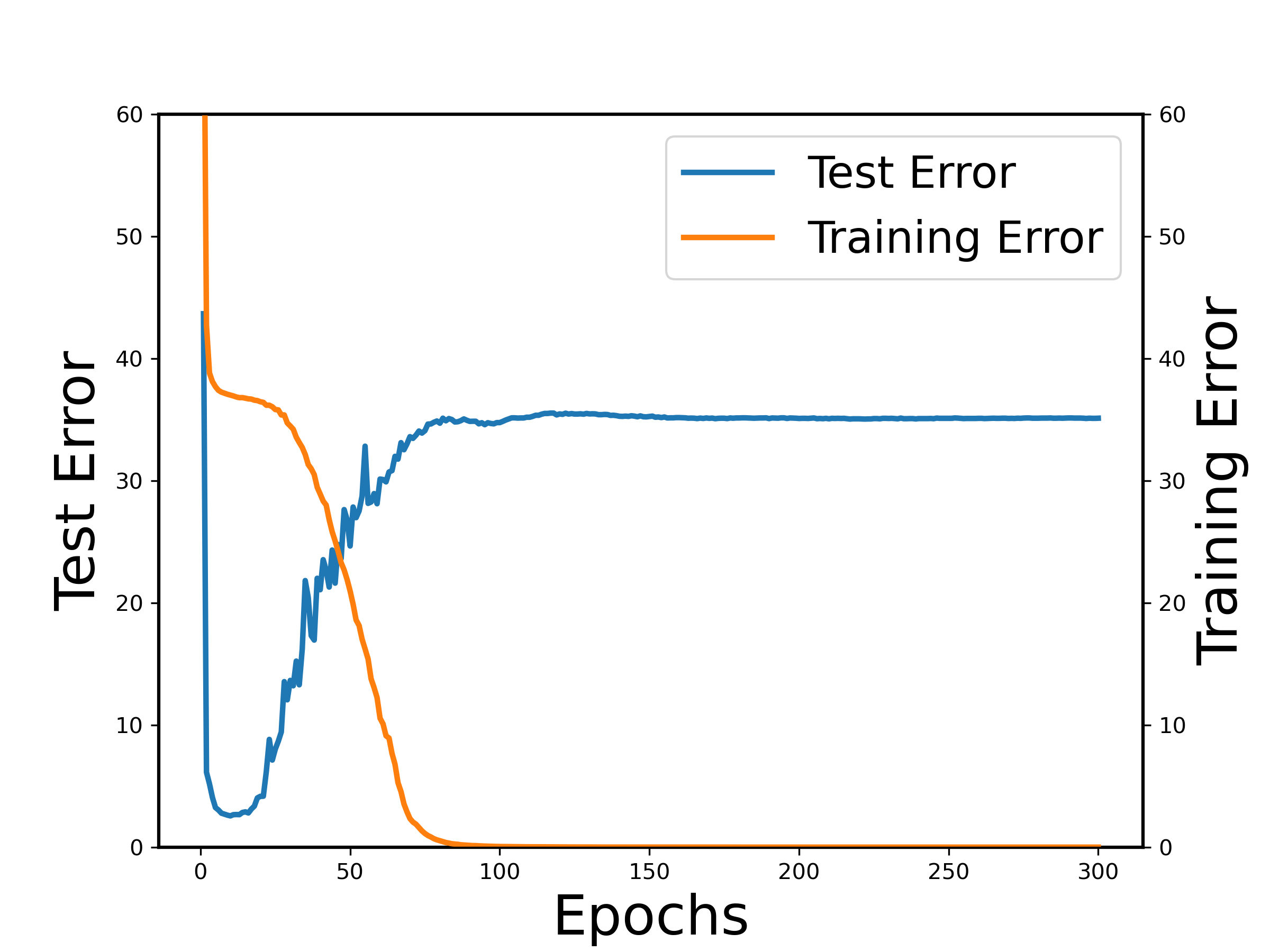}}
\end{minipage}
}
\quad
\subfigure[50\% label noise]{
\begin{minipage}{0.29\linewidth}
\centerline{\includegraphics[width=1\textwidth]{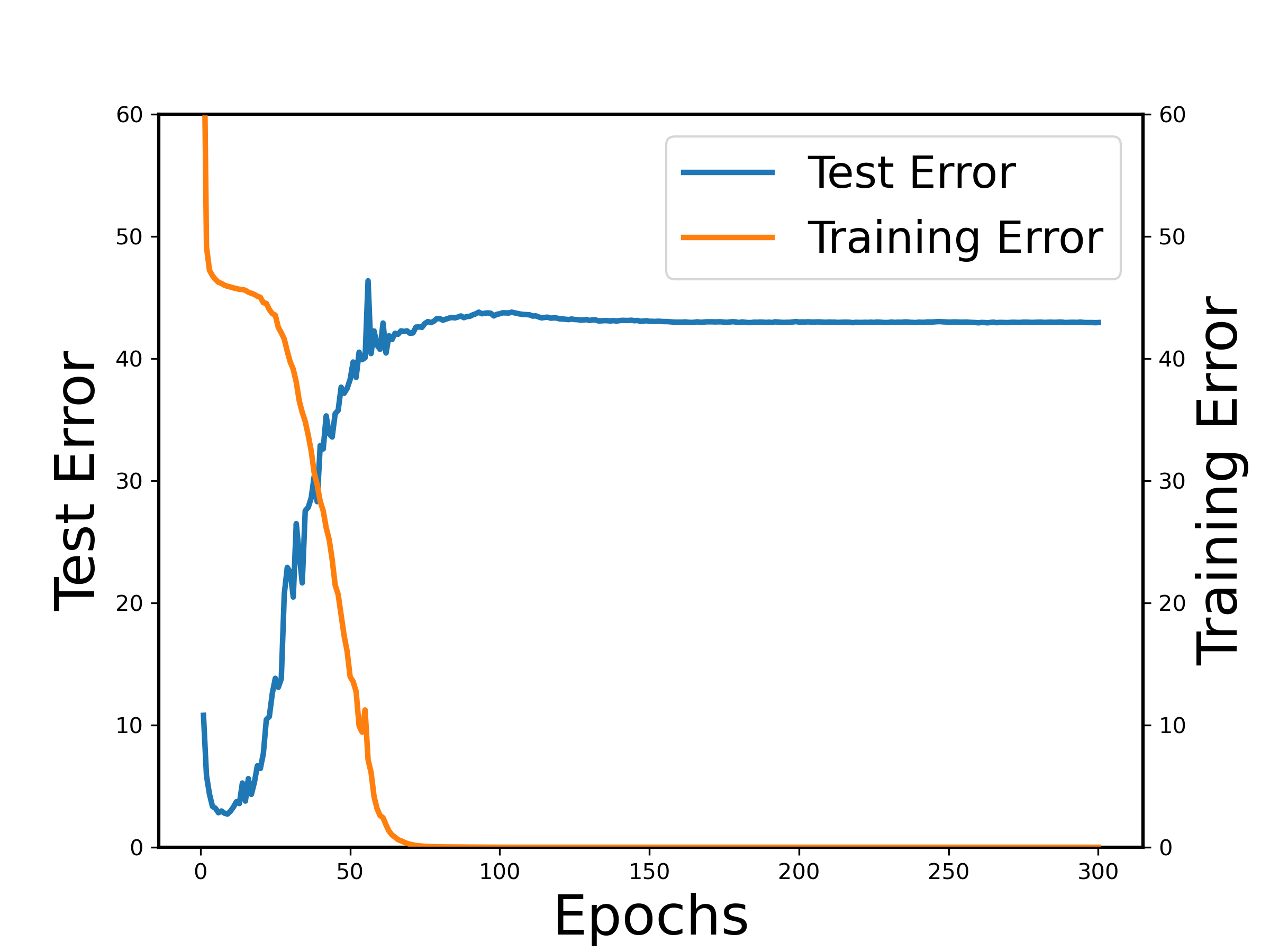}}
\end{minipage}}

\caption{\textbf{The training plot for corrupted MNIST with different levels of label noise using SGD.} Figure (c) is copied from Figure \ref{fig:intro}.}
\label{fig:mnist_appen}
\end{figure*}

\begin{figure*}[t]  
\centering
{
\begin{minipage}{0.29\linewidth}
\centerline{\includegraphics[width=1\textwidth]{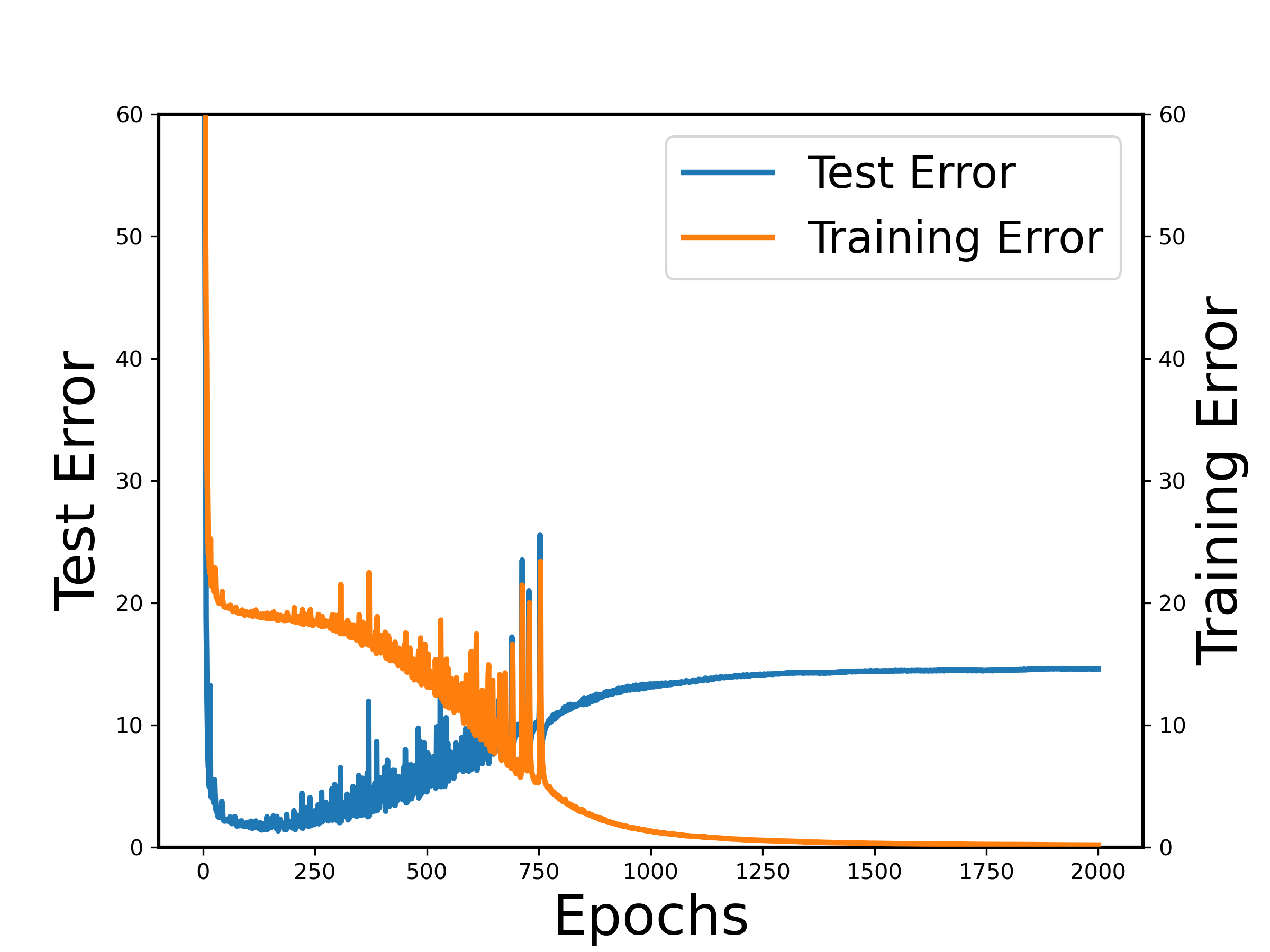}}
\end{minipage}
}

\caption{\textbf{The training plot for corrupted MNIST with 20\% label noise using GD.} }
\label{fig:mnist_gd}
\end{figure*}

\begin{table}[t]
\caption{
\textbf{The test error of optimal stopping iterate and final iterate on MNIST dataset with different levels of label noise.}
The results demonstrate that stopping iterate can have significantly better generalization performance than interpolating solutions for real datasets.}
\label{tab: mnist error}
\vskip 0.15in
\begin{center}
\begin{small}
\begin{sc}
\begin{tabular}{ccc}
\toprule
Noise Level & Optimal Test Error& Final Test Error \\
\midrule
$0\%$&  $1.07\%$& $1.13\% $ \\
$10\%$&  $1.75\%$& $10.16\% $ \\
$20\%$&  $2.88\%$& $19.90\% $ \\
$30\%$&  $2.18\%$& $27.94\% $ \\
$40\%$&  $2.57\%$& $35.15\% $ \\
$50\%$&  $2.71\%$& $42.95\% $ \\

\bottomrule
\end{tabular}
\end{sc}
\end{small}
\end{center}
\vskip -0.1in
\end{table}